\documentclass[graybox]{svmult}

\usepackage{type1cm}        
\usepackage{makeidx}         
\usepackage{graphicx}        
\usepackage{multicol}        
\usepackage[bottom]{footmisc}

\usepackage{newtxtext}       %
\usepackage{newtxmath}       

\usepackage{amsmath}
\usepackage{algpseudocode}
\usepackage[Algorithm,ruled]{algorithm}
\usepackage{overpic}
\usepackage{pgfplots}
\usepackage{pgf,tikz}

\newcommand{\IR}{\mathbb{R}}
\newcommand{\IS}{\mathbb{S}}
\newcommand{\cC}{\mathcal{C}}
\newcommand{\cE}{\mathcal{E}}
\newcommand{\cF}{\mathcal{F}}
\newcommand{\cG}{\mathcal{G}}
\newcommand{\cO}{\mathcal{O}}
\newcommand{\cR}{\mathcal{R}}
\newcommand{\cV}{\mathcal{V}}
\newcommand{\ds}{\text{ds}}
\newcommand{\dx}{\text{dx}}

\newcommand{\norm}[1]{\ensuremath{\left\|#1\right\|}}

\newcommand{\dist}{\text{dist}}
\newcommand{\diam}{diam}
\newcommand{\length}{length}

\renewcommand{\phi}{\varphi}

\newcommand{\ie}{\emph{i.e.}}
\newcommand{\etal}{\emph{et al.}}

\newcommand{\zorah}[1]{\textcolor{black}{#1}}

\renewcommand{\algorithmicrequire}{\textbf{Input:}}
\renewcommand{\algorithmicensure}{\textbf{Output:}}

\makeindex             
\begin{document}

\title*{Efficient 2D-to-3D Deformable Shape Matching for 3D Shape Retrieval Applications}
\titlerunning{Efficient 2D-to-3D Deformable Shape Matching}
\author{Zorah L\"ahner, Emanuele Rodol\`{a}, Frank R. Schmidt, Michael M. Bronstein and Daniel Cremers}
\authorrunning{L\"ahner et al.}
\institute{Zorah L\"ahner \at Universit\"at Siegen, Department Elektrotechnik und Informatik, H\"olderlinstr. 3, 57076 Siegen, Germany, \email{zorah.laehner@uni-siegen.de}
\and Emanuele Rodol\`a \at
Sapienza Universit\`a di Roma, Dipartimento di Informatica, Via Salaria 113 III Piano, 00198 Roma, Italy, \email{rodola@di.uniroma1.it}
\and Frank R. Schmidt \at Bosch Center for Artifical Intelligence,
Robert-Bosch-Campus 1, 71272 Renningen, Germany, \email{frank.r.schmidt@de.bosch.com}
\and Michael M. Bronstein \at
Universit\`a della Svizzera Italiana, Faculty of Informatics SI-109,
Via Giuseppe Buffi 13, 6904 Lugano, Switzerland, \email{michael.bronstein@usi.ch}
\and Daniel Cremers \at
Technische Universit\"at M\"unchen, Informatik 9, Boltzmannstrasse 3, 85748 Garching, Germany, \email{cremers@tum.de}
}
%

  \maketitle

\abstract{We discuss an algorithm for non-rigid 2D-to-3D shape matching, where
  the input is a 2D query shape as well as a 3D target shape and the
  output is a continuous matching curve represented as a closed
  contour on the 3D shape. We cast the problem as finding the shortest
  circular path on the product 3-manifold of the two shapes.  We prove
  that the optimal matching can be computed in polynomial time with a
  (worst-case) complexity of $\cO(mn^2\log(n))$ where $m$ and $n$
  denote the number of vertices on the 2D and the 3D shape
  respectively. Computing a solution with a relative error of $\epsilon$
  has the same complexity but
  ensures faster convergence in some cases.
  Quantitative evaluation confirms that the method provides excellent
  results for sketch-based deformable 3D shape retrieval.}

  \section{Introduction}\label{sec:introduction}

  The last decade has witnessed a tremendous growth in the quantity and
  quality of geometric data available in the public domain. One of the
  driving forces of this growth has been the development in 3D sensing
  and printing technology, bringing affordable sensors such as Microsoft
  Kinect or Intel RealSense and 3D printers such as MakerBot to the mass
  market.
  The availability of large geometric datasets brings forth the need to
  explore, organize, and search in 3D shape collections, ideally in the
  same easy and efficient way as modern search engines allow to process
  text documents.
  \begin{figure}[t]
  \begin{center}
    \includegraphics[width=0.24\linewidth]{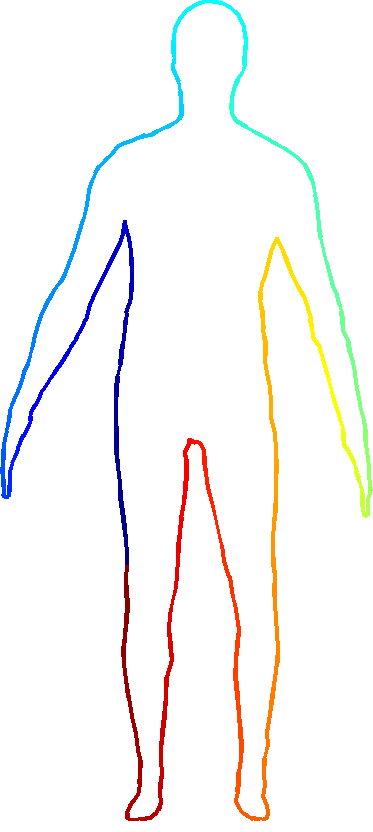}\hspace{8pt}
    \includegraphics[width=0.31\linewidth]{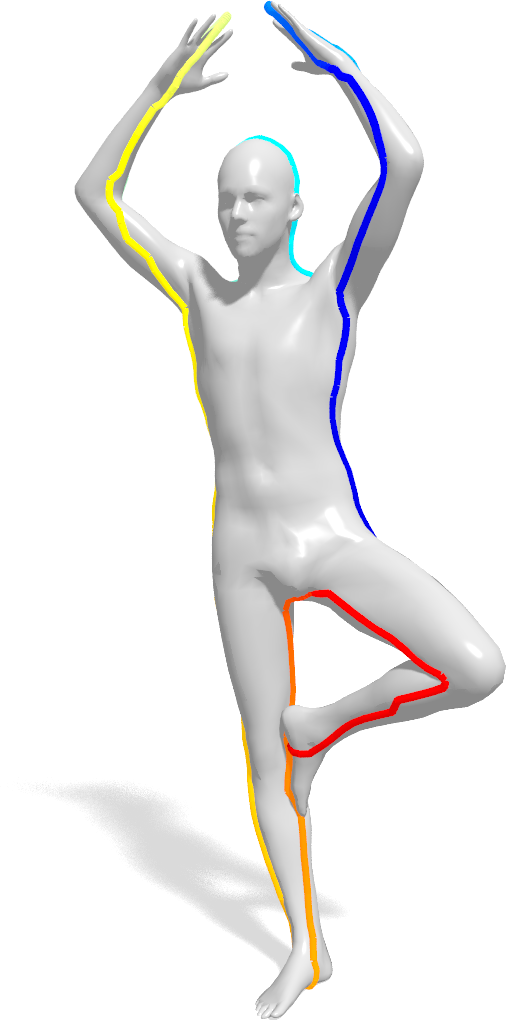}
    \includegraphics[width=0.31\linewidth]{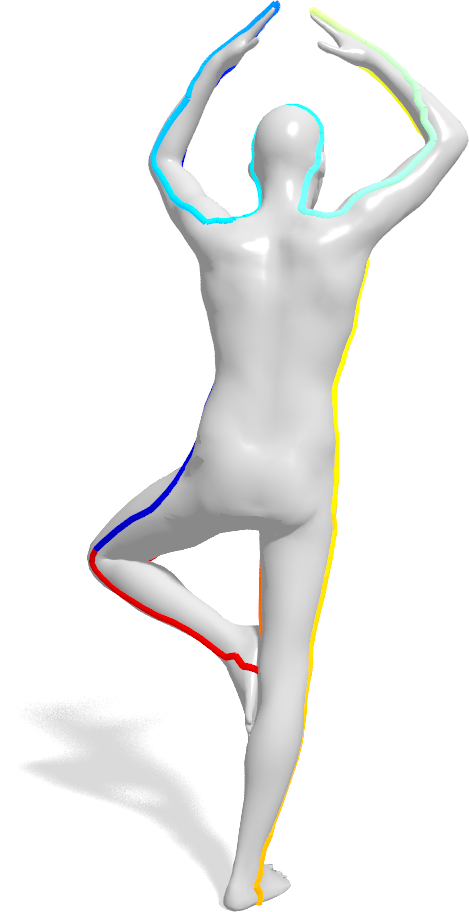}
  \end{center}
  \caption{\label{fig:teaser}We propose a shape matching method
    between a 2D query shape~(left) and a 3D target shape~(right), both
    of which are allowed to deform non-rigidly.  The globally optimal
    matching~(shown on top of the 3D target) is guaranteed to be
    continuous.}
  \end{figure}

  Numerous works on {\em content-based shape retrieval}
  \cite{funkhouser2003search,tangelder2008survey,bronstein2011shape} try
  to extend popular search paradigms to a setting where the 3D query
  shape is matched to shapes in the database using some criterion of
  geometric similarities.
  Typically, a 3D shape is represented as a descriptor vector
  aggregating some local geometric features, and retrieval is done
  efficiently by comparing such vectors \cite{bronstein2011shape}.
  However, the need for the query to be a 3D shape significantly limits
  the practical usefulness of such search engines: non-expert human
  users are typically not very skilled with 3D modeling, and thus
  providing a good query example can be challenging.
  As an alternative to 3D-to-3D shape retrieval, several recent works
  proposed 2D-to-3D or {\em sketch-based} shape retrieval, where the
  query is a 2D image representing the projection or the silhouette of a
  3D shape as seen from some viewpoint
  \cite{eitz2012sbsr,furuya14,Li20151,su15mvcnn,hueting2015crosslink}.
  This setting is much more natural to human users who in most cases are
  capable of sketching a 2D drawing of the query shape; however, the
  underlying problem of `multi-modal' similarity between a 3D object and
  its 2D representation is a very challenging one, especially if one
  desires to deal with non-rigid shapes such as human body poses. In
  fact, so far all methods for 2D-to-3D matching have limited the
  attention to rigid shapes such as chairs, cars, etc.

  \begin{figure*}[t]
  \begin{center}
    \includegraphics[width=0.25\linewidth]{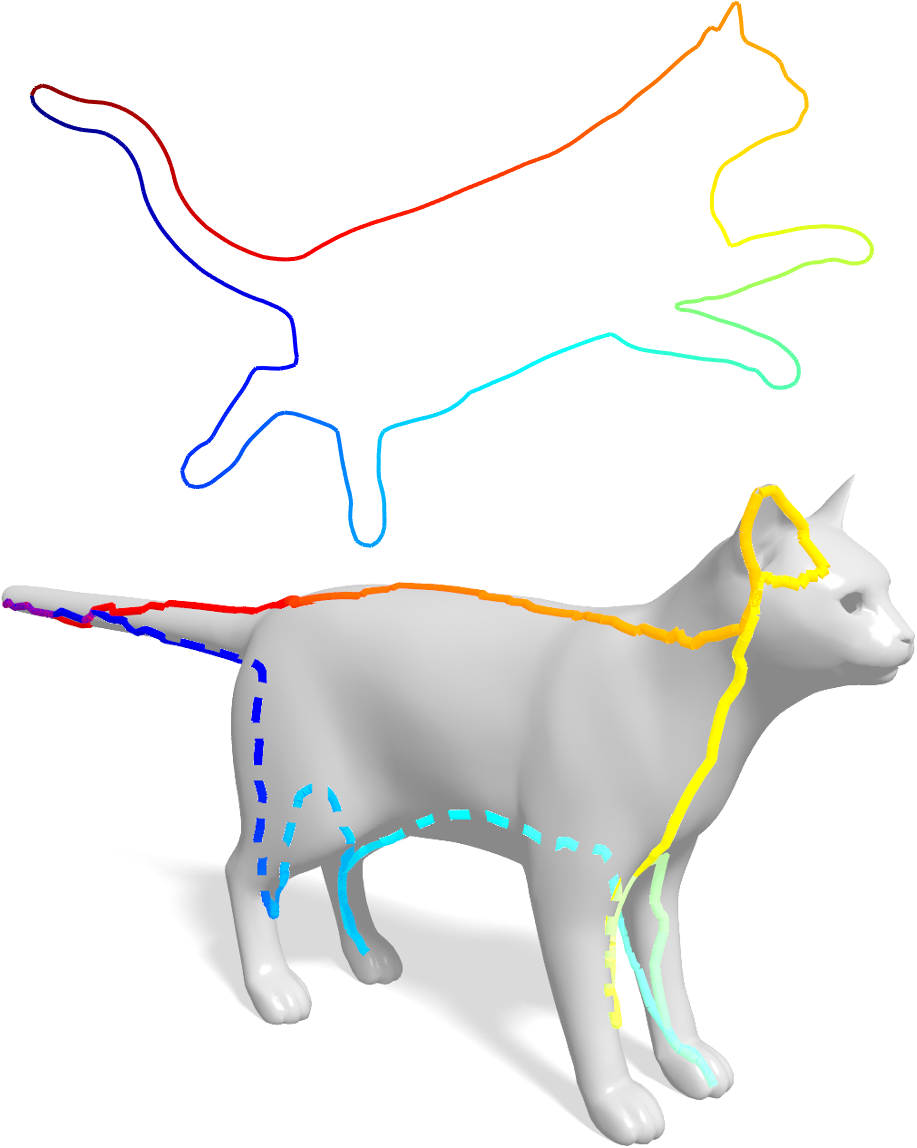}\hspace{12pt}
    \includegraphics[width=0.16\linewidth]{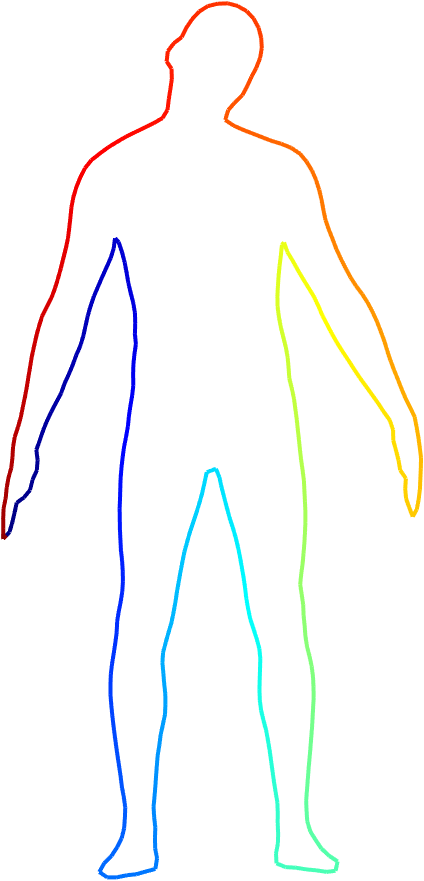}
    \includegraphics[width=0.147\linewidth]{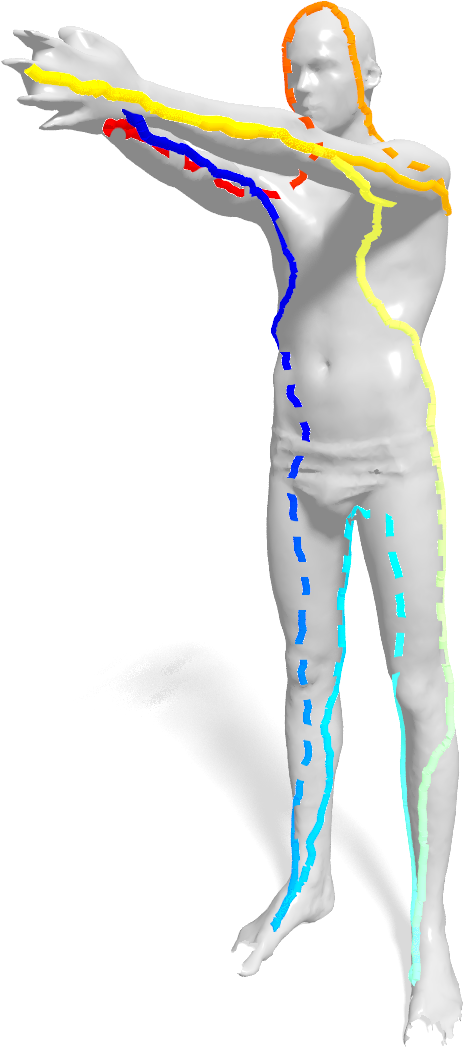}\hspace{12pt}
    \includegraphics[width=0.189\linewidth]{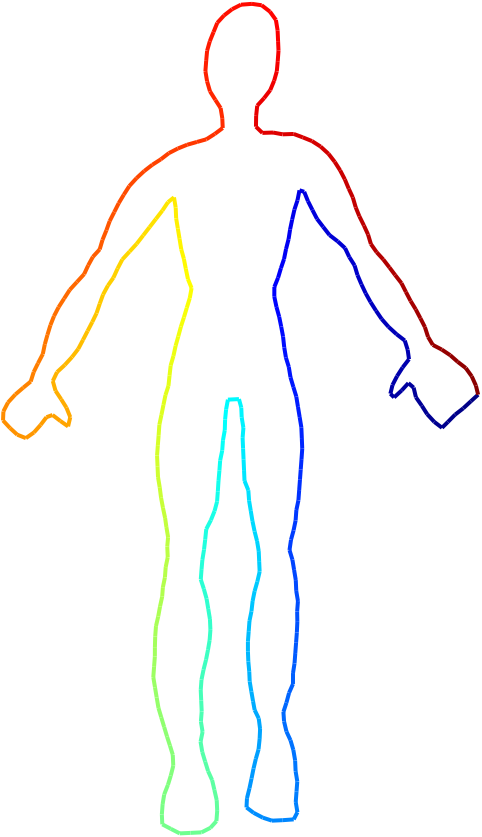}
    \includegraphics[width=0.16\linewidth]{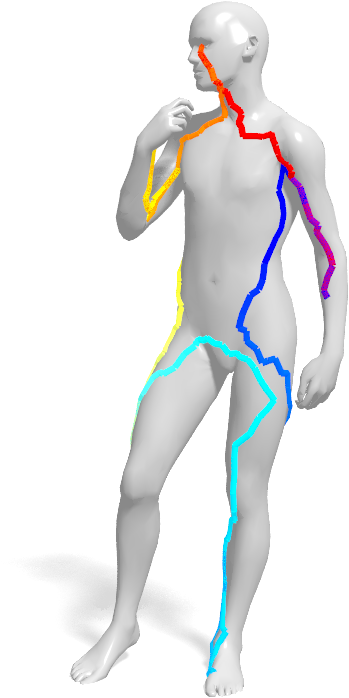}
  \end{center}
  \caption{\label{fig:example_matches}
  Examples of correspondences between a cat contour and a 3D cat (left) and human contours and 3D humans (middle and right).
  The human shape (middle) exhibit a topological change along the
    hands which is handled well by our method. The contour on the right is handdrawn.
    The dashed lines denotes correspondences that are normally not visible from this
    point of view. }
  \end{figure*}

  Here, we propose a method for automatically finding
  correspondence between 2D and 3D deformable objects.
  To the best of our knowledge, this is the first method to address the
  problem in the deformable setting.
  Prior approaches to the 2D-to-3D matching problem were limited to the
  less challenging rigid setting.
  The input to our algorithm is a 2D query curve and a 3D target
  surface, and the output is the corresponding continuous curve on the
  surface (see Fig.~\ref{fig:teaser}). Our method can either guarantee a globally
  optimal solution or an $\epsilon$-approximation of it, and has polynomial time
  complexity. The $\epsilon$-tight
  solution comes with decreased runtime and slightly less precision but
  this does not actually corrupt the retrieval results too much. In practice, we observe computationally efficiency with 1-20 seconds
  for running the complete pipeline to got a matching of shapes up to $40,000$ vertices.
  Additionally, our approach allows using different local feature
  descriptors for 2D and 3D data.  In particular, we show how spectral
  2D and 3D features can be compared such that we obtain a semantically
  driven matching between 2D and 3D shapes.
  As a byproduct of the correspondence we also get a 2D-to-3D similarity
  criterion allowing efficient sketch-based shape retrieval.

  The rest of this chapter is organized as follows.  In the remaining part
  of this section, we review previous works and summarize our
  contributions.
  In Section~\ref{sec:matching}, we formulate the 2D-to-3D matching
  problem as an energy minimization problem and discuss its
  discretization and optimization.
  Section~\ref{sec:results} shows experimental results.  We consider a
  challenging application, namely deformable sketch-based retrieval.
  Finally, Section~\ref{sec:conclusion} concludes by discussing
  the limitations and potential extensions of the approach.

  \paragraph{Differences to Conference Version.}~
  Additionally to the full content of the conference version \cite{LRSBC16},
  this version contains a more detailed theoretical derivation of the algorithm in
  Section~\ref{sec:matching}.  In addition to the global
  optimal method, we propose an additional
  version to compute an $\epsilon$-close approximate solution that is much faster
  and in practice comparable to the exact computation. The dataset for the
  retrieval experiments was extended to include three more classes and
  more than 30 poses for already existing classes. We also added two handdrawn
  2D queries to the retrieval dataset.  Furthermore, the runtime
  experiments were extended to show the quadratic dependency on the
  number of vertices in the 3D target shape. Last, we show the energies of the
  retrieval experiments form class-wise clusters when embedded into 2D.

  \section{Related Work}\label{sec:relatedwork}

  \subsection{2D and 3D shape correspondence.}\label{subsec:rw:2D3D}
  The classical 2D-to-2D and 3D-to-3D settings of the shape matching
  problem have been thoroughly researched in the computer vision and
  graphics communities (see~\cite{van2011survey} for a survey).
  In the domain of 3D-to-3D shape matching, a major challenge is to have
  theoretical guarantees about the optimality and quality of the
  correspondence.
  Several popular methods try to find a correspondence that minimally
  distorts intrinsic distances between pairs of corresponding points by
  approximate solution to the quadratic assignment problem
  \cite{leordeanu2005spectral,memoli2005theoretical,bronstein2006generalized,rodola12cvpr,rodola13elastic}.
  A recent line of works builds upon the functional
  representation~\cite{ovs12,pokrass13,rodola-cvpr14,kovn15,rodola15},
  where a point-wise map is replaced by a linear map between function
  spaces. While these approaches solve many challenging problems, they
  lack theoretical guarantees on the quality of the final solution.  In
  particular, none of these methods yields provably {\em continuous}
  maps between the given shapes. %
  This seems to be surprising, since the functional maps are linear.
  Nonetheless, linear operators between function spaces do not need to be
  continuous. Hence, finding continuous matchings is still a challenging
  task in the area of 3D-to-3D matching applications.  A notable
  exception is represented by the work of Windheuser
  \etal~\cite{wind11}.
  The authors mathematically define the space of continuous face-wise correspondences between two shapes including degenerations which allow for arbitrary stretching and contractions.
  This formulation in combination with a physically meaningful deformation cost is converted into an Integer Linear Program (ILP) whose result is the least-cost, continuous correspondence.
   Similarly to~\cite{wind11}, our method comes with
  the theoretical guarantee of a continuous solution.  In contrast
  to~\cite{wind11}, our method can compute a matching in under half a
  minute instead of several hours.
  A related challenge is to generate a 3D model from a human produced contour drawing \cite{10.1145/3240508.3240699}.
  The setting differs in that no 3D shape is given as input, but nevertheless it is necessary to
  generate some kind of correspondence of the 2D input to the result to guarantee
  that the input was properly transferred \cite{Lift3D_SA20}.
  In our 2D-to-3D correspondence problem, the 2D shape is modeled as a
  closed planar curve and the 3D shape as a surface in $\IR^3$. To find
  the correspondence, we look for a closed curve on the surface.
  From this perspective, our method can be seen as an extension of an
  image segmentation task that looks for a closed curve within a 2D
  image domain.  It was shown in~\cite{Schoenemann-Cremers-pami10} that
  this segmentation problem can be formulated as finding a shortest path
  in the product graph of the 2D image domain and the 1D curve domain,
  where the size of the graph depends on the Lipschitz constant of the
  mapping. The drawback is that this constant is typically unknown in
  advance.
  Differently from~\cite{Schoenemann-Cremers-pami10}, the size of the
  constructed graph with our method is independent of the Lipschitz
  constant.

  Furthermore, one of the main challenges in our method is to find an
  initial match on the product graph.
  In~\cite{Schoenemann-Cremers-pami10} this problem was solved by
  parallelization. As a result, the overall computation time is not
  reduced but just distributed intelligently among several
  computational cores.
  Instead, we use a branch-and-bound approach that only computes
  shortest paths in those regions that are `most promising'.  This
  strategy reduces the runtime substantially (especially with
  well-chosen shape descriptors), while still converging to a global
  optimum.

  Even in the simpler 2D-to-2D setting, the computation of a
  globally-optimal correspondence can be very slow if we do not know an
  initial match.  For example, the runtime of Dynamic Time Warping
  methods is $\cO(n^2)$ if an initial match is given, and $\cO(n^3)$ if
  every possible initial match is tested independently, where $n$ is the
  number of shape samples.
  It was shown that by exploiting the planarity of the involved graph,
  the runtime of the whole matching including an initial match can
  be reduced to $\cO(n^2\log(n))$ by using shortest circular path or
  graph cut approaches~\cite{Maes-91,Schmidt-et-al-iccv07,Schmidt-et-al-cvpr09}.  A competitive approach is the
  branch-and-bound approach of~\cite{Appleton-2003}. While this does not
  reduce the worst case time complexity of $\cO(n^3)$, it is rather fast
  in practice. Since this method does not use the planarity of the involved
  graph, we can adapt it to our scenario in order to reduce the
  practical runtime substantially.

  \subsection{Sketch-based retrieval.}~\label{subsec:rw:sketchbased}
  One of the important applications of 2D-to-3D matching is
  shape-from-sketch retrieval. This problem has recently drawn the
  attention of the machine learning community as a fertile playground
  for cross-modal feature
  learning~\cite{eitz2012sbsr,furuya14,Li20151,su15mvcnn,hueting2015crosslink}.
  Herzog \etal~\cite{herzog15} recently proposed to learn a shared
  semantic space from multiple annotated databases, on which a metric
  that links semantically similar objects represented in different
  modalities (namely 2D drawings and 3D targets) is learned.
  More recent approaches learn an expressive representation from a collection
  of images to retrieve the corresponding 3D shape \cite{ZHOU2017101,Xu_2019_ICCV}.
  \cite{Nie2021} applies a metric learning approach to embed the pairs of inputs in
  the same space, and use the projection in this space as the retrieval feature.
  A similar approach is taken by \cite{TABIA201724} which is not restricted to images
  but can be generalized to any kind of different modality, including contours.
  The method of \cite{Yang20} specifically tackles the problem of retrieval from
  a collection of silhouettes by learning comparable features from a convolutional neural network
  and making use of the ability to project the object onto the silhouette.

  Although these approaches yield promising results in the {\em rigid} setting and
  can address some variability of the shapes, its applicability to the
  {\em non-rigid} setting is an open question.
  In contrast, our method targets explicitly the setting when both the
  3D target and the 2D query are allowed to deform in a non-rigid
  fashion.
  Furthermore, the method of~\cite{herzog15} as well as other existing
  approaches mostly focus on finding {\em similarity} between a 2D
  sketch and a 3D shape while we solve the more difficult problem of
  finding {\em correspondence} (from which a criterion of similarity is
  obtained as a byproduct).

  \section{2D-to-3D Matching}\label{sec:matching}

  In this section, we formulate, discretize and optimize the shape
  matching problem between a 2D \emph{query shape} and a 3D
  \emph{target shape}. We start with a continuous formulation
  and discretize it in Section~\ref{sec:optimization}.
  By a \emph{shape} we refer to the outer shell of
  an object. The object itself will be referred to as the shape's
  \emph{solid}. The 3D ball $B=\{x\in\IR^3|\norm{x}\leq1\}$ is for
  example the solid of the sphere $\IS^2=\{x\in\IR^3|\norm{x}=1\}$.  We
  summarize this convention in the following definition:

  \begin{definition}\label{def:shape}
    A compact set $S\subset\IR^d$ is called a \emph{shape} of dimension
    $d$ if it is a connected, smooth manifold and if it can be
    represented as the boundary $S=\partial U$ of an open subset
    $U\subset\IR^d$. In this case, we call $U$ the \emph{solid} of $S$.
  \end{definition}

  Note that this definition implies that a 3D shape is a 2-manifold and
  a 2D shape is a 1-manifold and shapes can not have self-intersections.

  This section is organized as follows. In Section~\ref{sec:energy} we
  will cast the 2D-to-3D shape matching problem as an energy
  minimization problem, which we will globally optimize and approximate
  in Section~\ref{sec:optimization}. To this end, we assume that
  descriptive features for both shapes are given and that it is possible
  to measure the dissimilarity between 2D and 3D features. The specific
  choice of such features depends on the application. For the
  application of \emph{shape retrieval} that we discuss in
  Section~\ref{sec:retrieval} we use purely spectral features.

  \subsection{Energy formulation}\label{sec:energy}
  Given the 2D query shape $M\subset\IR^2$ and the 3D target shape
  $N\subset\IR^3$, we search a continuous mapping $\phi\colon M\to N$.
  $\phi$ is a proper \emph{2D-to-3D matching} if it is an immersion, \ie, if the
  differential $d\phi$ is of maximal rank at any point.
  This implies that the mapping itself
  is differentiable and cannot collapse into a single point anywhere, \ie, for
  any $x, y \in M$ such that $\phi(x) = \phi(y) = q$ there has to exist a point
  $z \in M$ in between $x, y$ for which $\phi(z) \neq q$ holds.

  The goal of this approach is
  to find a 2D-to-3D matching that sets points that look alike into
  correspondence.  To this end let $f_M\colon M\to\IR^{k_M}$ and
  $f_N\colon N\to\IR^{k_N}$ be two different feature maps. We want to
  stress that the dimensions $k_M$ and $k_N$ do not need to agree. In
  order to measure the dissimilarity between the 2D feature $f_M(x)$ of
  $x\in M$ and the 3D feature $f_N(y)$ of $y\in N$, we assume a positive
  distance function $\dist\colon\IR^{k_M}\times\IR^{k_N}\to\IR_0^+$ to
  be given. This distance takes care of the difficult task of comparing
  2D features with 3D features but depends of course on the chosen
  features. A concrete choice for the features and distance functions
  is presented in Section~\ref{para:features}. Given the two feature maps $f_M$
  and $f_N$ as well as the distance
  function $\dist$, we call a 2D-to-3D matching $\phi$ optimal if it
  minimizes the energy
  \begin{align}\label{eq:E_line}
    E(\phi) \colon= \int_{\Gamma_\phi}\dist(f_M(s_1),f_N(s_2))\ \ds,
  \end{align}
  where $\Gamma_\phi=\{(s_1,s_2)\in M\times N|s_2=\phi(s_1)\}$ denotes
  the graph of $\phi$, the submanifold of the product manifold $M \times N$ which includes all the pairs that are set in correspondence by $\phi$.
  The energy $E$ accumulates the distance between the features
  of any pair of matched points $(s_1, s_2)$ in $\phi$ which means a good solution will place points of the query onto points on the 3D shapes with a similar features.
  Since $\phi$ is assumed to be continuous, this is not a simple nearest neighbor problem but takes the geometry of $M$ into account.
  Note that $\Gamma_\phi$ is a simplicial complex
  due to the immersion property of $\phi$. $E$ is therefore defined as a
  line integral.
  Calculating the area elements needed for the line integral on $\Gamma_\phi$ is not straight-forward.
  Instead we substitute $\phi$ with a higher-dimensional mapping $\hat\phi: M \to M \times N$, $x \mapsto (x, \phi(x))$. $\hat\phi$ conveys the same information as $\phi$ but, as we will see below, integrates on $M$ with known area elements.
  The definition of $\hat\phi$ leads to the following equations $s = (s_1, s_2) = (x, \phi(x))$, $\hat\phi(M) = \Gamma_\phi$ and therefore $f_N(s_2) = f_N \circ \phi(x)$. We apply these in the substitution rule with $\hat\phi$ which results in the following energy function:

  \begin{align}
    E(\hat\phi) &= \int_M \dist(f_M(x),f_N\circ\phi(x)) \cdot  \Vert \hat\phi'(x) \Vert\ \dx
    \label{eq:substituded}
  \end{align}

  Since $M$ is a one-dimensional manifold, the norm can be calculated as $\Vert \hat\phi'(x)\Vert = \sqrt{d\hat\phi^\top d\hat\phi}$. Fortunately, $d\hat\phi$ only depends on $d\phi(x)$ with an additional constant entry from where $x$ was mapped to itself:
  \begin{align}
    d\hat\phi(x): &\ T_xM \to T_xM \times T_yN \\
    &\ v \mapsto \begin{pmatrix} v \\ d\phi(x)v \end{pmatrix} = \begin{pmatrix} 1 \\ d\phi(x) \end{pmatrix} v
  \end{align}
  Including this in Eq.~\eqref{eq:substituded} leads to the following energy function which depends only on $\phi$ again:
  \begin{align}\label{eq:E_coord}
    E(\phi) &= \int_M \dist(f_M(x),f_N\circ\phi(x)) \sqrt{1+d\phi_x^{\top}d\phi_x}\ \dx
  \end{align}

  Hence, the energy $E$ can be broken down into the \emph{data~term}
  $\dist(f_M(\cdot),f_N\circ\phi(\cdot))$ comparing feature values and the
  \emph{regularizer} $\sqrt{1+d\phi^{\top}d\phi}$ penalizing stretching.

  \begin{description}
    \item{\textbf{Regularization.}}
        If we ignore the data term, the global minimum of $E$ would result
      in a constant $\phi$. This $\phi$ is continuous, but matches every
      point on $M$ to the same point on $N$.  It therefore ignores the
      similarity information stored in the data term.

    \item{\textbf{Data term.}}
    If we ignore the regularizer, the global minimum of $E$ can be
    computed by selecting for each $x\in M$ a $y\in N$ that minimizes the
    given feature distance $\dist(f_M(x),f_N(y))$. In this case, the minimizer
    of $E$ will match similar points but $\phi$ might be neither injective
    nor continuous. Combining the data term with the
    regularization results in a smooth matching function $\phi$ that also
    takes similarity into account.

    \begin{figure*}[t]
    \centering
      \includegraphics[width=.3\linewidth]{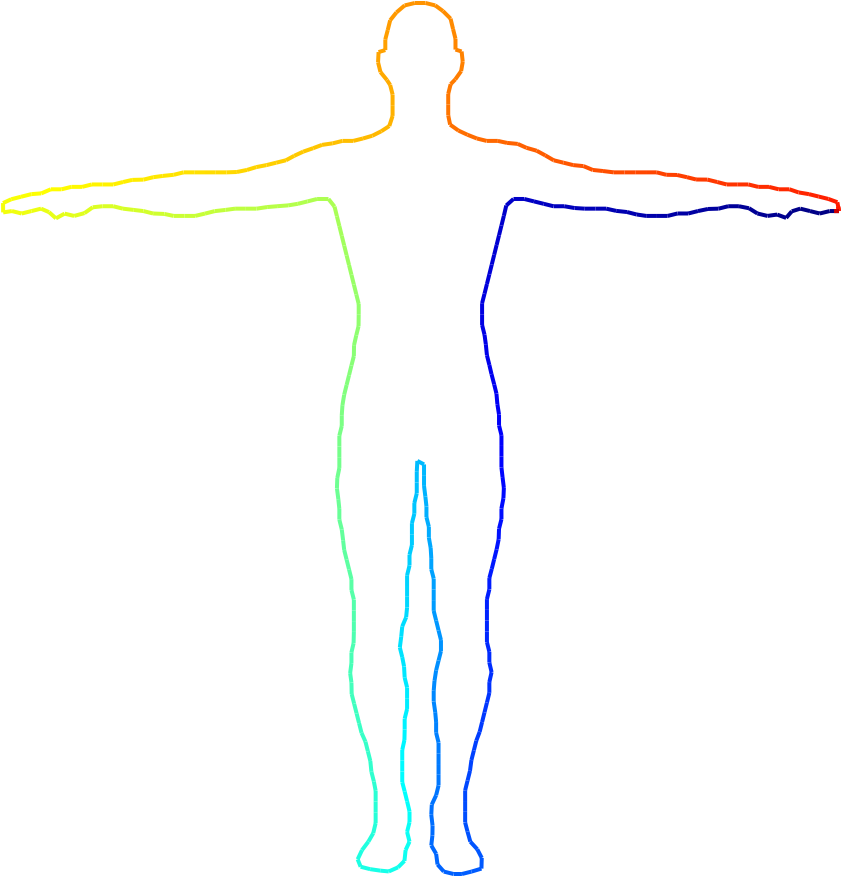}
      \includegraphics[width=.3\linewidth]{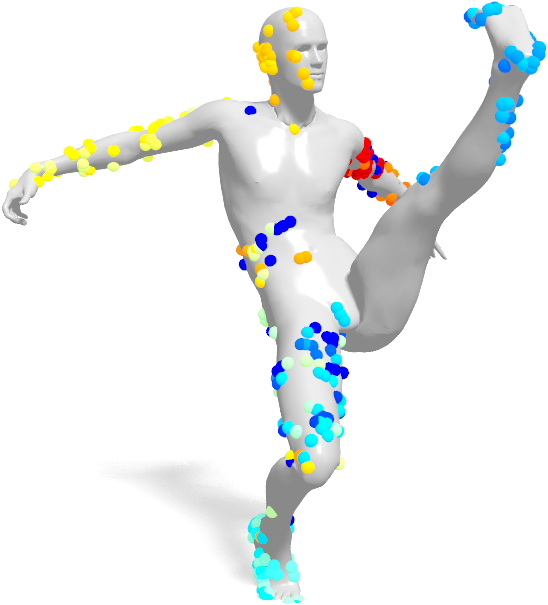}
      \includegraphics[width=.3\linewidth]{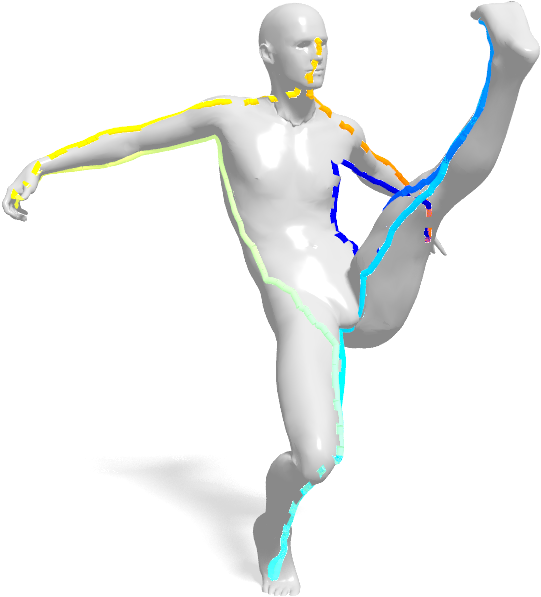}
      \caption{Matching between a 2D query shape (left) and a 3D target shape achieved
      by solving a LAP between the same point-wise features our method uses (middle)
      and our method (right).}
      \label{fig:lap}
    \end{figure*}

  \end{description}

  Alternatively to the energy described here, one could also choose to
  enforce injectivity of $\phi\colon M\to N$. This would lead to a
  linear assignment problem (LAP), which normally results in
  non-continuous matchings and is rather slow. If the shapes $M$ and $N$
  are discretized at $m$ and $n$ points, respectively, the overall
  runtime of the Hungarian method~\cite{munkres} to solve this problem
  is $\cO(n^3)$. The method that we propose does not only provide for a
  smooth and continuous solution, but also has a better worst
  case runtime complexity than the LAP (cf.  Theorem~\ref{thm:runtime}). %
  Exploring the runtime of the LAP approach for one matching instance
  resulted in a runtime of 11 hours instead of just a few seconds for
  our method. The LAP matching result using the same features as in our
  experiments can be seen in Figure~\ref{fig:lap}. %

    \subsection{Optimization}\label{sec:optimization}

    So far we defined the energy we want to minimize. In the following,
    we address the discretization of this minimization problem and
    show that a globally optimal solution and an approximate solution
    within an arbitrary margin of the optimum can be computed
    efficiently by solving a \emph{shortest path problem} on the discrete product
    manifold of both shapes.
    In practice, we often observed that many solutions with an energy close
    to the optimum are not qualitatively worse than the optimum. At the same time
    the algorithm finds solution close to the optimum very fast and spends a
    proportionally long time on finding the exact optimum. For this reason we
    implemented an approximate variation of our algorithm that uses an upper and
    lower bound to identify how close to the optimum the current solution is and
    stops if they are reasonably close.

    To this end we assume that the 2D query shape $M$ and the 3D target
    shape $N$ are discretized. Thus, $M$ is given as a simple, directed circular
    graph, \ie, $\cG_M=(\cV_M,\cE_M)$ with
    \begin{align*}
      \cV_M=&\{x_0,\ldots,x_{m-1}\}\subset\IR^2\\
      \cE_M=&\{(x_i,x_j)\in \cV_M^2 ~|~ %
      j\equiv i+1\mod{m}\}.
    \end{align*}

    The target shape $N$ on the other hand is given as a 3D triangular mesh
    $\cG_N=(\cV_N,\cE_N,\cF_N)$, where the set of vertices is
    denoted by $\cV_N=\{y_0,\ldots,y_{n-1}\}$
    , $\cE_N$ is the set of unoriented edges,
    and $\cF_N$ the set of faces.  Further, we assume that the feature
    maps $f_M\colon M\to\IR^{k_M}$ and $f_N\colon N\to\IR^{k_N}$
    with a corresponding distance function are
    given as information on the vertices.  Thus, we can represent
    the feature distance between all possible matches as a matrix $D\in\IR^{m\times n}$ with
    $D_{ij}=\dist(f_M(x_i), f_N(y_j))$ which can be precomputed.

    Given this discretization, we define the product graph $\cG_{M\times
      N}=(\cV_{M\times N},\cE_{M\times N},\cC_{M\times N})$ via
    \begin{align*}
      \cV_{M\times N}=&\{0,\ldots,m-1\}\times\{0,\ldots,n-1\}\\
      \cE_{M\times N}=& %
      \left\{[({i_0},{j_0}),({i_1},{j_1})]\in\cV_{M\times N}^2\right|\\
      &\quad({i_1}={i_0})\land(y_{j_1},y_{j_0})\in\cE_N \text{ or }\\
      &\quad({i_1}={i_0}+1)\land({j_1}={j_0}) \text{ or }\\
      &\quad({i_1}={i_0}+1)\land(y_{j_1},y_{j_0})\in\cE_N \}
    \end{align*}

    \begin{figure}[t]
    \centering
      \begin{overpic}
        [trim=0cm 0cm 0cm -2cm,clip,width=\linewidth]{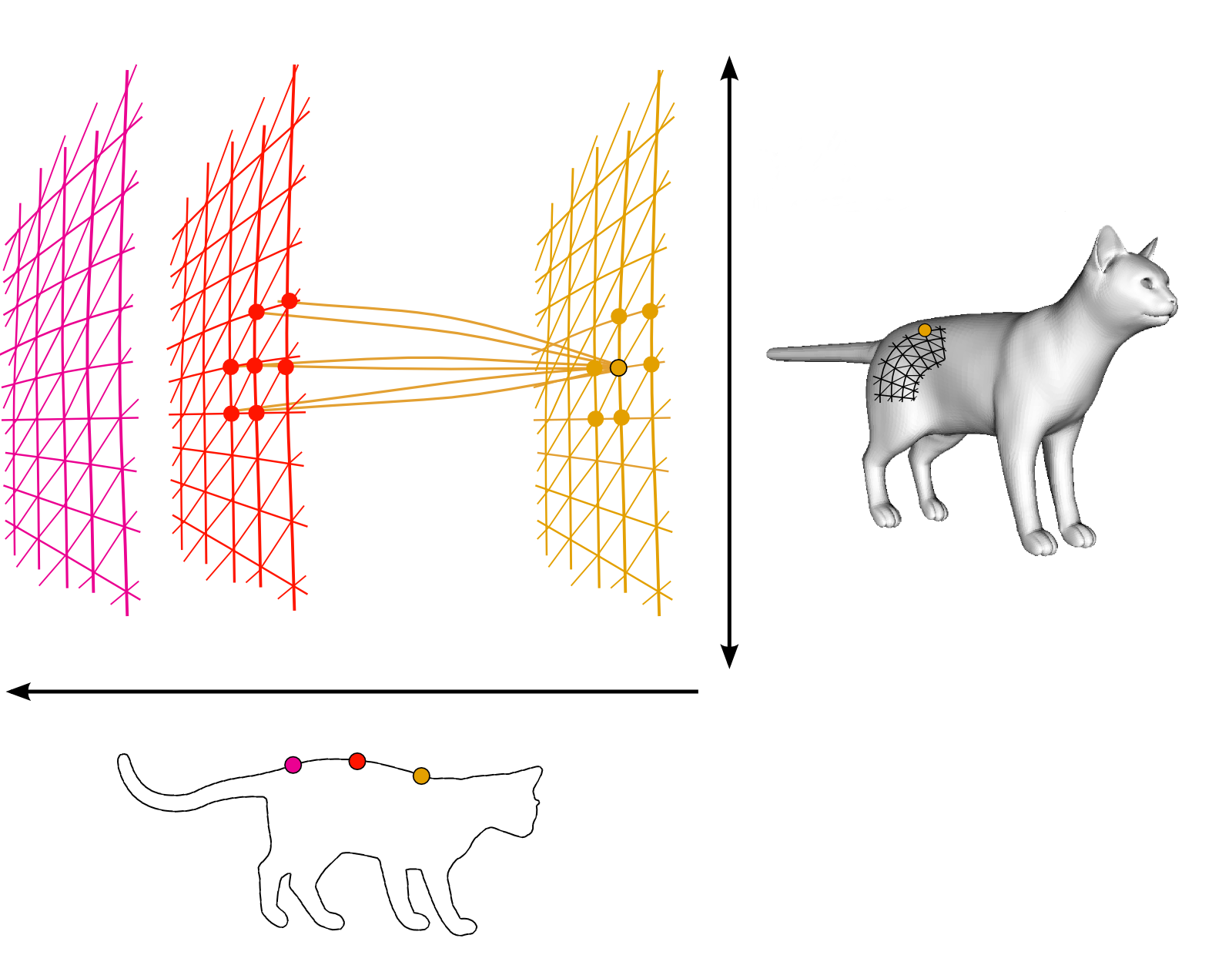}
        \put(25,75){$\cG_{M\times N}$}
        \put(47,25){$L_i$}
        \put(17,25){$L_{i+1}$}
        \put(3,25){$L_{i+2}$}
        \put(48,45){$(i,j)$}
        \put(24,45){$(i+1,k)$}
        \put(34,16){$i$}
        \put(10,9){$\cG_M$}
        \put(75,53){$j$}
        \put(70,30){$\cG_N$}
      \end{overpic}
      \caption{\label{fig:prodman} A node $(i,j)$ in the product graph
        $\cG_{M\times N}$ represents a match between the vertex
        $i\in\cV_M$ of the contour $M$ and the vertex $j\in\cV_N$ of the
        surface $N$. All feasible matches with respect to vertex $i$ form
        the layer $L_i=\{i\}\times \cV_N$. Edges are defined between a
        node $(i,j)$ and $(i+1,j)$ as well as $(i,k)$ and $(i+1,k)$ for
        all surface vertices $k\in\cV_N$ that are adjacent to $j$. All
        these edges are directed and enforce a continuous matching. }
    \end{figure}

    The product graph takes the
    Cartesian product of the vertices in $\cG_M$ and $\cG_N$
    representing all possible point-wise matchings between the two
    shapes and connects elements of this set iff their projections on
    the original shapes are connected (or identical). See Fig.~\ref{fig:prodman}
    for an illustration. This preserves the original connectivity of each shape
    and enforces continuous solutions. Notice that the orientation of the circular 2D
    query shape is maintained because edges with $i_0 \neq i_1$ are
    directed. Edges where $i_0 = i_1$ are connected in both directions.

    Furthermore, we will be interested in \emph{edge-wise}
    instead of point-wise costs to solve the shortest path problem.
    We use the distances as stored in $D$ at each endpoint of
    one edge $\left[ (i_0, j_0), (i_1, j_1)\right] \in \cE_{M\times N}$

    \begin{align*}
      D_{i_0,j_0}=&
      \ \dist\left(f_M(x_{i_0}),f_N(y_{j_0})\right)\\
      D_{i_1,j_1}=&
      \ \dist\left(f_M(x_{i_1}),f_N(y_{j_1})\right)
    \end{align*}

    and linearly interpolate the costs along the edge. This is equal to
    integrating over the average of both values and results in the cost function
    \begin{align*}
      \cC_{M\times N}[({i_0}&,{j_0}),({i_1},{j_1})]=\\
      &\frac{D_{i_0,j_0}+D_{i_1,j_1}}
      2\cdot\norm{(x_{i_0},y_{j_0})-(x_{i_1},y_{j_1})}.%
    \end{align*}
     $(x,y)$ is a 5D-coordinate with the stacked coordinate
     values from $x \in M$ (2D) and $y \in N$ (3D). This is a discretization of the line integral between both vertices from
     Equation \ref{eq:E_line}.

      To solve the shortest path problem a fixed source and target set
      is needed. We have no information about which vertices are contained in
      the solution but to have a circular path, we know that each $x \in M$ has
      to represented at least once in the solution. Therefore,
      the representation of the 2D shape $M$ is cut at an arbitrary $x$  and is extended by
      having two copies of $x_0$, namely at position $i=0$ and at
      position $i=m$. As a result, any continuous matching can be re\-presented by
      a path from $(0,j)$ to $(m,j)$. Hence, an optimal matching can be
      cast as finding a shortest path in a graph if an initial match
      $(x_0,y_j)\in\Gamma_\phi$ is given. Such a computation can easily
      be done by Dijkstra's algorithm \cite{dijkstra59}. Using a priority heap
      the computation takes
      $\cO(mn\cdot\log(mn))$ steps. Since there is no path
      from $(i_1,j_1)$ to $(i_0,j_0)$ if $i_1>i_0$, we associate to each
      \emph{layer} $\{i\}\times\{0,\ldots,n-1\}$ a different priority
      heap and reduce the runtime to $\cO(mn\log(n))$. These
      observations lead to the following theorem considering the above mentioned
      observations and the fact that we have to test $n$ different
      initial matches.

      \begin{theorem}\label{thm:runtime}
        Given a 2D query shape $M$ and a 3D target shape $N$,
        discretized by $m$ and $n$ vertices, respectively, we can find a
        minimizer of~\eqref{eq:E_line} in $\cO(mn^2\log(n))$
        steps. If $n=\cO(m^2)$, this leads to the subcubic
        runtime of $\cO(n^{2.5}\log(n))$.
      \end{theorem}

      This theorem shows that we can find a globally optimal matching in
      polynomial time. Nonetheless, this may still lead to a high
      runtime  since we have to find for each vertex
      $y\in\cV_N$ a shortest path in $\cG_{M\times N}$. In order to
      circumvent this problem we follow a branch-and-bound strategy inspired
      by the method of~\cite{Appleton-2003}.

      \begin{figure*}
        \quad\includegraphics[trim={8cm .3cm 8cm 0cm},clip,width=.23\linewidth]{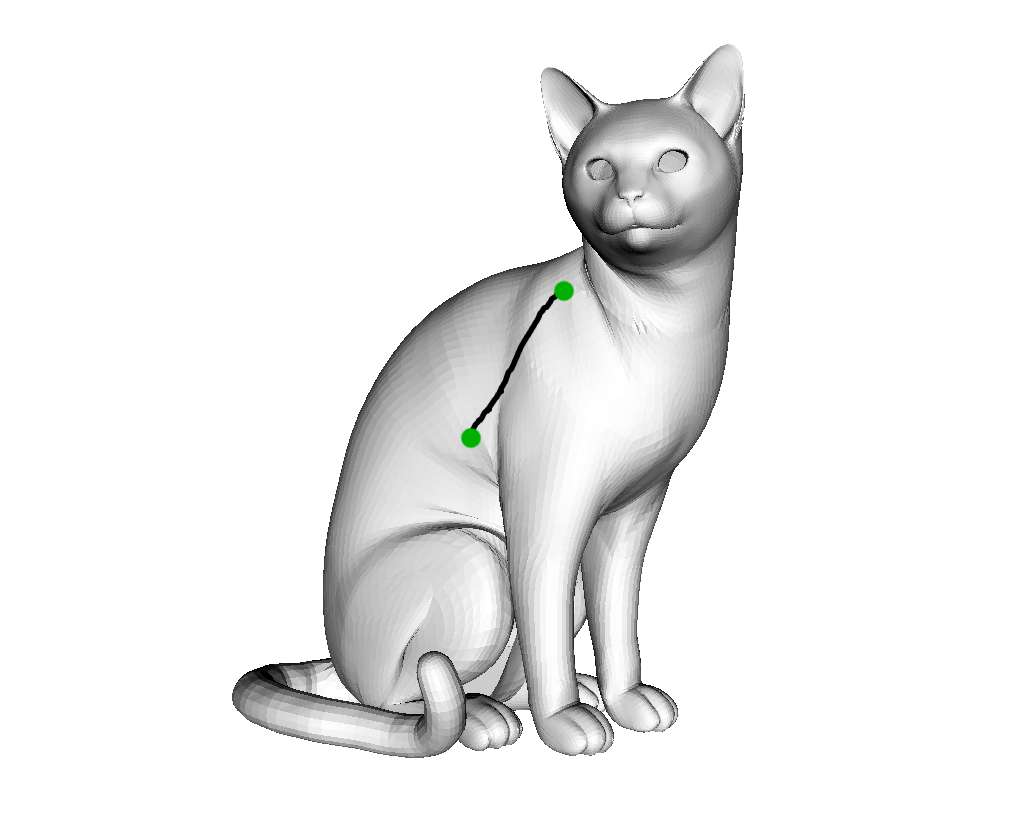}
        \includegraphics[trim={8cm .3cm 8cm 0cm},clip,width=.23\linewidth]{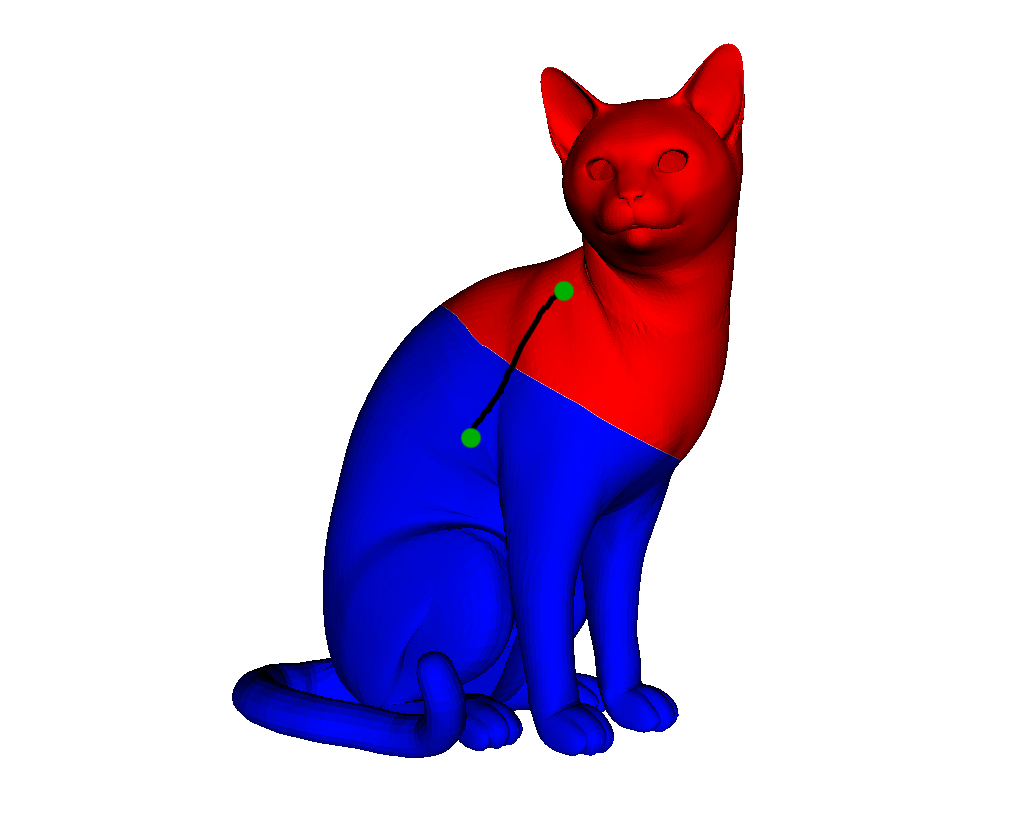}
        \includegraphics[trim={8cm .3cm 8cm 0cm},clip,width=.23\linewidth]{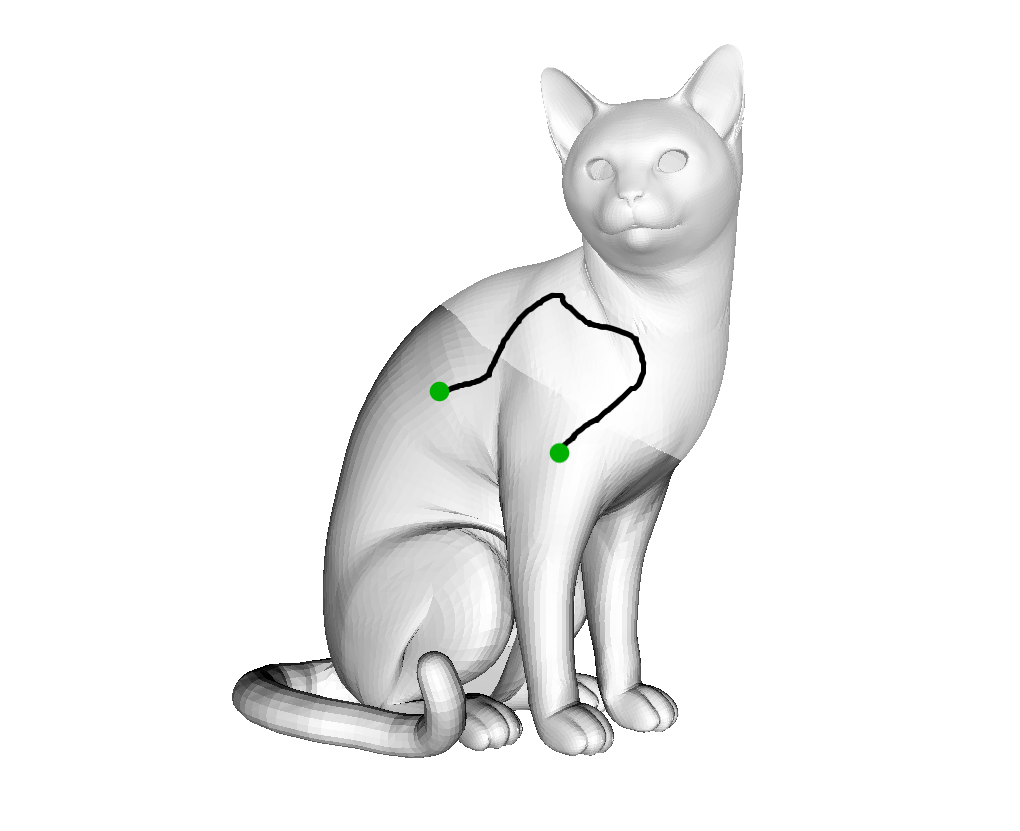}
        \includegraphics[trim={8cm .3cm 8cm 0cm},clip,width=.23\linewidth]{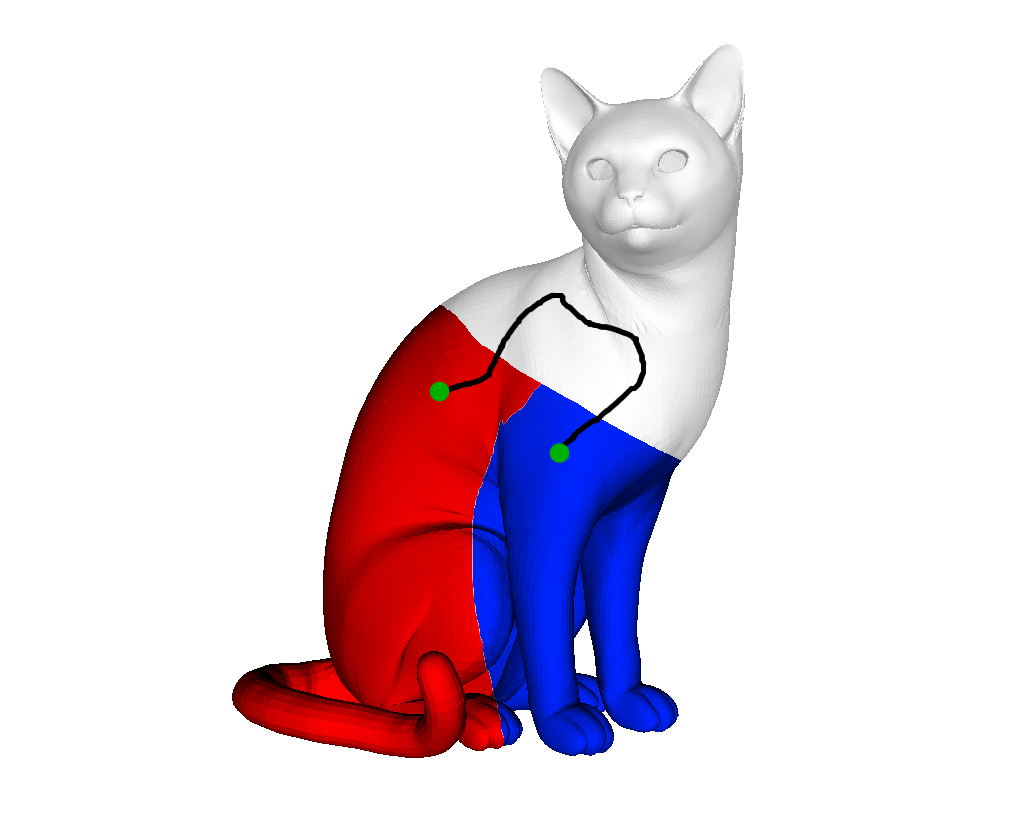}
        \caption{Exemplary first two iterations of the branch-and-bound algorithm. Paths are projected
        from the product manifold on the 3D shape. (Left) Absolute energy minimal path $p_1$ through the cut product manifold.
        (Middle Left) 3D Shape is separated in two subareas. One containing the source and the other the sink point of $p_1$.
        (Middle Right) In the next iteration one area is chosen; blue in this example but red will be processed later in the same way.
        The source and sink of the optimal path $p_2$ have to lie in the blue area. The
        rest of the path can go anywhere. This excludes the previous minimum $p_1$. (Right) If the source and sink of $p_2$
        do not conincide, the blue area is again split in two. }
        \label{fig:bandb}
      \end{figure*}

      \begin{algorithm}
      \begin{algorithmic}[h]
        \Require{$\cG_{M\times N}=(\cV_{M\times N},\cE_{M\times
            N},\cC_{M\times N})$, \textcolor{blue}{$\epsilon$}}%
        \Ensure{Matching path $\Gamma_\phi$}%
        \State{Let $R:=\{0,\ldots,n-1\}$ ;}
        \State{Define $\cR=\{R\}$ and $b\colon\cR\to\IR$ via $b(R)=0$ ;}
        \State{Define \verb+isFound+=\verb+false+ ;}
        \While{\texttt{isFound}=\texttt{false}}
        \State{%
          \State{\texttt{lowerbound}=$\min b$ ;}
          \State{Let $R\in\arg\min b$; $\cR=\cR\setminus\{R\}$ ;}
          \State{Find shortest path $\Gamma$ in $\cG_{M\times N}$\\
          \hspace{2em}from $\{0\}\times R$ to $\{m\}\times R$ ;}
          \State{$\Gamma$ is a path from $(0,i)$ to $(m,j)$ ;}
          \If{$i=j$}\State{%
            \texttt{isFound}=\texttt{true};%
            $\Gamma_\phi=\Gamma$%
          }\Else{%
            \State{\textcolor{blue}{Find indices $l, k \in R$ with minimal Euclidean
            distance such that the shortest path} }
            \State{\textcolor{blue}{in $\cG_{M\times N}$ to $(m,k)$
            originates in $(0,l)$ ;}}
            \State{\textcolor{blue}{\texttt{upperbound} = $dist( (0,k), (m,l) )
              + dist( (m,l), (m,k) )$ in $\cG_{M\times
                N}$\; } \textcolor{blue}{\If{\texttt{lowerbound} $\geq
                (1 - \epsilon) \cdot$ \texttt{upperbound}}\State{
                $\Gamma_\phi= dist( (0,l), (m,l) )$ ; \texttt{break} ;
              }\EndIf} }
            \State{Divide $R$ into $R=R_1\cup R_2$ such that}
            \State{\hspace{2em}$x\in R_1$: }
            \State{$\Leftrightarrow\dist_N(x,i)<\dist_N(x,j)$ ;}
            \State{Set $\cR:=\cR\cup\{R_1,R_2\}$ ;}
            \State{Set $b(R_1)=b(R_2)=\length(\Gamma)$ ;}
            \If{$\length(\Gamma) \leq \min b$}\State{
              \texttt{isFound}=\texttt{true} ;%
            }\EndIf
          } \EndIf
        } \EndWhile

      \end{algorithmic}

      \caption{%
          2D-to-3D matching via branch-and-bound. %
          The blue code denotes the addition needed for the approximation algorithm.
          \label{algo:BB}
        }

      \end{algorithm}
      The main idea is to follow a coarse-to-fine strategy. First we
      compute the shortest path between the sets $\{0\}\times R$ and $\{m\}\times
      R$, which connects $(0,i)$ with $(m,j)$. In the first iteration,
      we set $R = \cV_N$.  In this case $i$ and $j$ do not need to
      coincide but the energy of this path provides a lower bound on the
      optimal closed solution.  If this path connects the corresponding
      points, \ie, $i=j$, we found a valid path. Otherwise, we separate
      the region $R$ into two sub-regions, $R_1$ containing $i$ and
      $R_2$ containing $j$, and recompute shortest paths starting
      and ending in $\{0\}\times R_i$ and $\{m\}\times R_i$ respectively. Since the shortest path of corresponding
      sub-regions will exclude the previously computed path, the
      optimal path in both $R_1$ and $R_2$ will have a larger energy
      than from the previously computed path. Thus, a natural order in
      which the subregions should be processed is induced.  We propose
      to separate $R$ with respect to the geodesic distance $\dist_N$ on
      the target shape $N$. We continue this process until we find our first matching path
      with $i = j$. Then, we
      still have to process those subdomains whose lower bound is
      smaller than the computed matching path. Afterwards, we are sure
      we found the globally optimal matching path $\Gamma_\phi$. See Figure~\ref{fig:bandb}
      for an illustration of one iteration.

      \paragraph{Approximation.\ }

      We observe that in practice already a few iterations in
      Algorithm~\ref{algo:BB} suffice if the features are
      distinctive and give a clear optimum. But in some cases, especially
      interclass matchings where the global optimum is not natural,
      the branch-and-bound is converging very slowy because many
      equally bad solutions exist. To circumvent
      this situation we implemented an extension using both an upper and lower
      bound of the optimal solution. If both are
      within a certain range of each other, the branch-and-bound stops. This
      guarantees a solution that is provably close to the optimum.

      The lower bound comes, as in the previous section, from paths with $i \neq j$.
      The upper bound is obtained by choosing the open path in the
      current region whose endpoints $(0, i)$ and $(m, j)$ have the smallest Euclidean distance
      between $i$ and $j$.
      We then extend the path on the product manifold via Dijkstra to be a closed path. Because
      each closed path is a valid solution the energy of this path is an upper
      bound for the global optimum. Notice that there are no guarantees on
      the tightness of this upper bound. If the upper bound is within some chosen
      $\epsilon$ of the lower bound from the previous region (see global optimization)
      we know we are close to the global optimum and stop.
      The $\epsilon$ can indicate the absolute or relative error but we choose it as
      the relative one to be insensitive to the scale of the energies.

      To keep the solution as close to the optimum as possible, we recompute
      the path between the $(0, i)$ and $(m, i)$ that closed the distance between the bounds.
      As a result from approximating the closing of the path from $(m, j)$, a different path between $(0, i)$
      and $(m, i)$ might have a lower energy. We keep the starting point as the initial match and recompute the shortest path in the
      product manifold. Because we can compute the optimum with respect to a given
      starting point the energy of this path will be lower or equal to the energy
      of the upper bound.

  \section{Applications}
  \label{sec:results}

  In this section, we apply the proposed method to the problem of
  sketch-based deformable shape retrieval.
  We emphasize that our method is parameter-free, and the only choice is
  with respect to the 2D and 3D features $f_M(\cdot), f_N(\cdot)$ as
  well as the distance function $\dist(\cdot,\cdot)$ between them.
  These choices depend on the specific application and help to
  illustrate the flexibility of our general 2D-to-3D shape matching
  approach.

  \subsection{Sketch-based shape retrieval}\label{sec:retrieval}

  We consider a particular shape retrieval setting in which the
  dataset is assumed to be a collection of 3D shapes, and the query is
  a 2D silhouette (possibly drawn by a human) represented by a closed
  planar curve. Differently from previous
  techniques~\cite{eitz2012sbsr,furuya14,Li20151,su15mvcnn}, our
  method does not use learning to compute features and most
  importantly, we allow the shapes to deform in a non-rigid fashion.
  The 2D-to-3D shape similarity is obtained by considering the
  minimal value of the energy $E(\phi^*)$ obtained by our optimization
  problem.

  \paragraph{Datasets.}

  Due to the novelty of the application, to date there is no benchmark
  available for evaluating deformable 2D-to-3D shape retrieval methods.
  We therefore construct such a benchmark using the FAUST \cite{bogo14},
  TOSCA \cite{bronstein08} and a subset of Non-rigid world \cite{bronstein2006nonrigid3d}
  (additional poses for TOSCA, gorilla, lioness and
  seahorse) datasets. The FAUST dataset consists of
  100 human shapes, subdivided into 10 classes (different individuals),
  each in 10 different poses. The latter two consist of over 100 shapes,
  subdivided into 12 classes (humans and animals in different poses).
  Shape sizes are fixed to around 7K (FAUST) and 10K (TOSCA) vertices.

  In FAUST each class comes with a `null' shape in a
  ``neutral pose'', where no deformation has been applied,
  which we use to define the 2D queries. To this end we cut each null
  shape across a plane of symmetry and project the resulting boundary
  onto a plane. This gives rise to 2D queries of 200-400 points on
  average.
  Note that by doing so we retain the ground truth point-to-point
  mapping between the resulting 2D silhouette and the originating 3D
  target. This allows us to define a quantitative measure on the quality
  of the 2D-to-3D matching between objects of the same
  class. In the extended dataset, silhouettes of one shape showing all
  important extremities of the class are produced and the 2D query is extracted
  from the binary image. Hence, no point-to-point but only class ground truths
  are available. The same method can be applied to hand-drawn silhouetts\footnote{The dataset containing shapes and matchings is publicly
    available at https://zorah.github.io/publication/2016-cvpr-efficient-globally-optimal-2d-to-3d-deformable-shape-matching}.

  As an addition to the queries produced through the actual 3D data, we drew a
  human by hand and used this as an additional query for the retrieval to show
  that the method also works on queries not produced using the targets.
  See Figure~\ref{fig:handdrawn} for the sketch and query.

  \begin{figure}[t]
  \centering
      \includegraphics[trim=0cm 0cm 0cm 0cm,clip,width=0.28\linewidth]{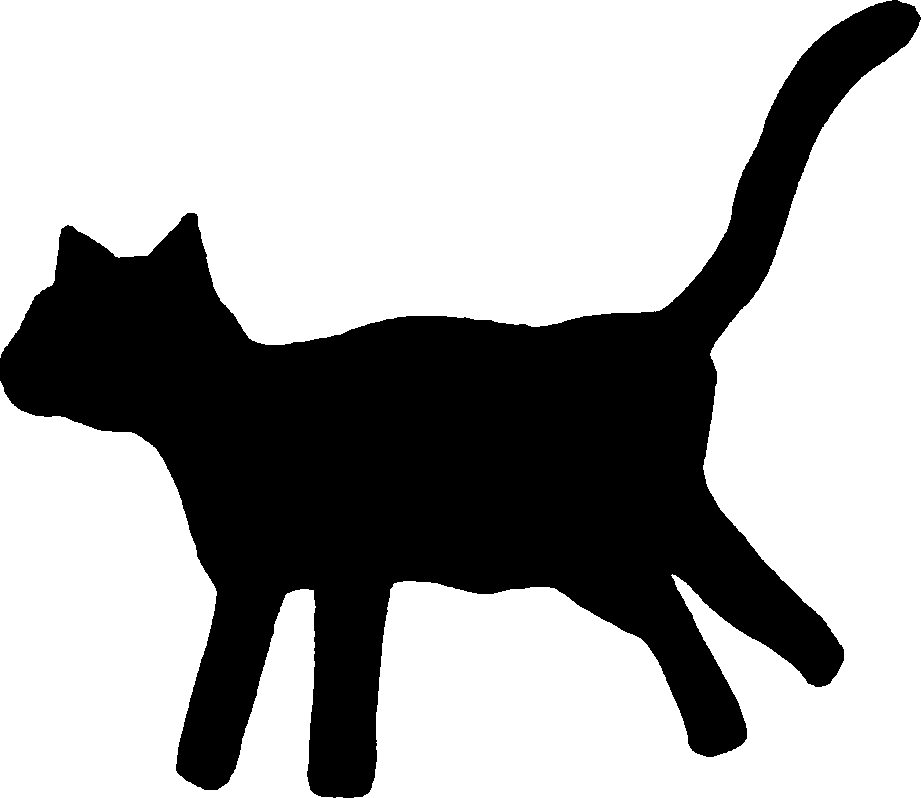}\quad
      \includegraphics[trim=0cm 0cm 0cm 0cm,clip,width=0.28\linewidth]{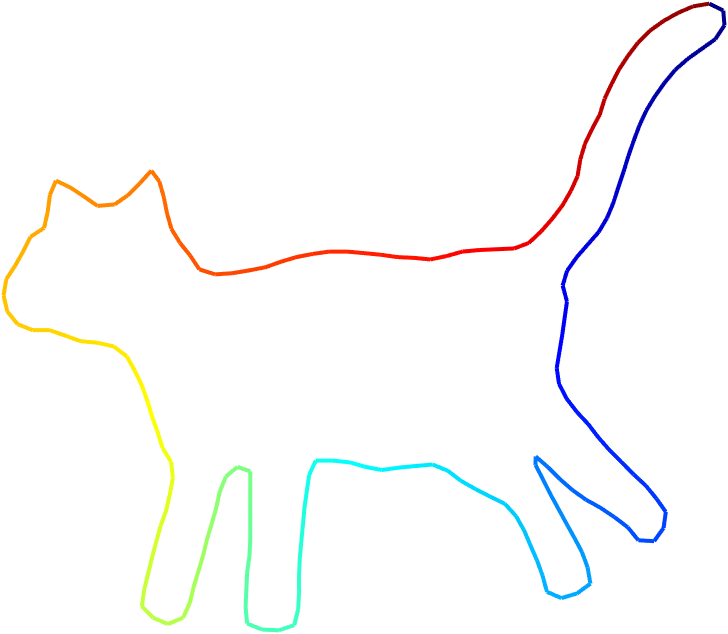}
      \includegraphics[trim=0cm 0cm 0cm 0cm,clip,width=0.4\linewidth]{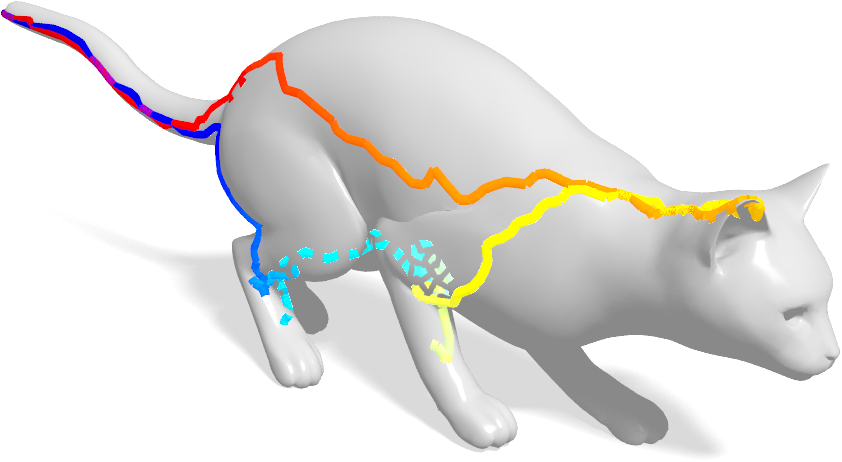}
    \caption{\label{fig:handdrawn} Handdrawn query for a cat shape and the extracted contour
    from the image. The image was simply thresholded to get a binary segmentation and then the
    contour extracted. Personal experience showed that drawing a solid silhouette is easier
    than directly drawing the contour. This also guarantees a shape as defined in Def.~\ref{def:shape}
    without self-intersection or other degenerations. }
  \end{figure}

  \paragraph{Error measure.}

    Let $M$ be a 2D shape represented as a planar curve and $N$ a 3D
    shape represented as a surface.  Let $\phi\colon M
    \to N$ be a matching between a 2D query shape and a 3D target shape,
    and let $\phi_0$ be the ground-truth matching. The matching error of
    $\phi$ at point $x\in M$ is given by
    \begin{align}\label{eq:err}
      \varepsilon_{\phi}(x) = \frac{\dist_N(\phi(x),\phi_0(x))}{\diam(N)},
    \end{align}
    where $\dist_N:N\times N\to\IR ^+_0$ denotes the geodesic distance on
    $N$ and $\diam(N)=\max_{x,y\in N}\dist_N(x,y)$ the geodesic diameter
    of $N$.  Note that due to the normalization, the values of the error
    $\varepsilon$ are within $[0,1]$.%


  \begin{figure}[t]
  \centering
    \begin{overpic}
      [trim=0cm 0cm 0cm 0cm,clip,width=\linewidth]{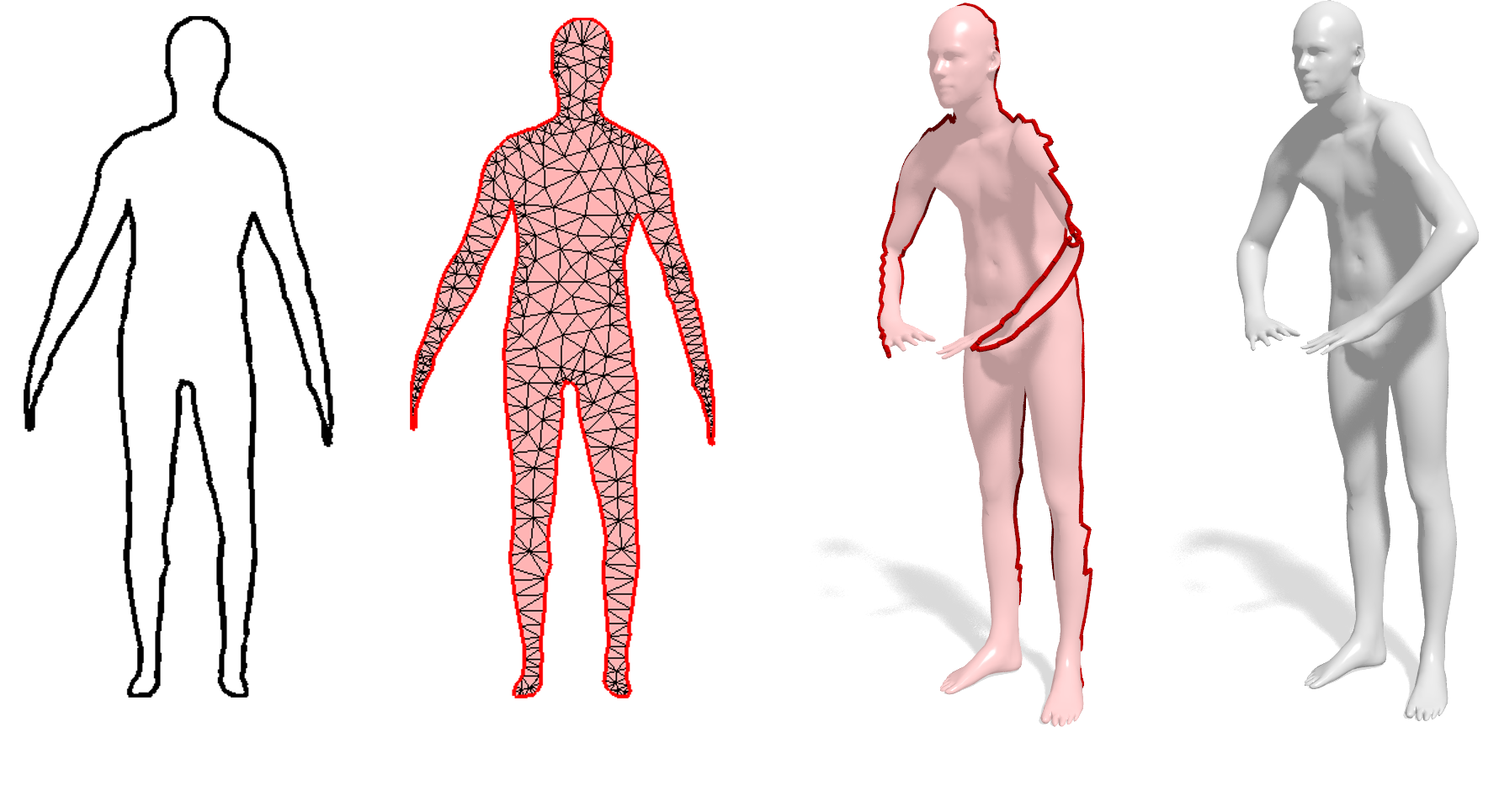}
      \put(4,1){\footnotesize (a) input 2D}
      \put(36.5,1){\footnotesize (b)}
      \put(65,1){\footnotesize (c)}
      \put(82,1){\footnotesize (d) input 3D}
      \put(49.4,25){\footnotesize isometric}
      \put(53.6,22){\footnotesize $\approx$}
    \end{overpic}
    \caption{\label{fig:flatness}Spectral features are constructed by
      considering the query shape (a) as the boundary of a 2D region
      (b), which is assumed to be a near-isometric deformation of a
      sub-region (c) of the 3D target shape (d). The 2D tessellation (b)
      is obtained via~\cite{shewchuk02}.}
  \end{figure}

  \paragraph{Spectral features.}
  \label{para:features}

  In this chapter we advocate the adoption of spectral quantities to
  define compatible features between 2D and 3D shapes. Note that differently from
  existing methods for 2D-to-3D matching, we compute local features independently for each
  given pair of shapes, \ie, no cross-modal metric learning is carried
  out.

  Let $\Delta_N$ be the symmetric Laplace-Beltrami operator on the
  3D shape $N$.  $\Delta_N$ admits an eigendecomposition with
  non-negative eigenvalues $0=\lambda_0<\lambda_1\leq\dots$ and
  corresponding orthogonal eigenfunctions $\psi_j\colon N\to\IR$ such
  that $\Delta_N\psi_j=\lambda_j\psi_j$. For the definition of our
  feature maps, 
  we consider multi-dimensional
  spectral descriptors, \ie, features constructed as functions of
  $\lambda_j$ and $\psi_j$. Note that, since $\Delta_N$ is invariant to
  isometric transformations of $N$, the derived spectral descriptors
  also inherit this invariance.

  Popular spectral features are the scaled Heat Kernel Signature
  (HKS)~\cite{sun09} and the Wave Kernel Signature (WKS)~\cite{aubry11},

  \begin{align}
   \label{eq:spectral}
   f^\mathrm{HKS}_k(x) &= \sum_{j\ge 1} e^{-t_k \lambda_j} \psi^2_j(x),\\
   f^\mathrm{WKS}_k(x) &= \sum_{j\ge 1} e^{-\frac{(\log e_k - \log
      \lambda_j)^2}{2\sigma^2}} \psi^2_j(x),
  \end{align}
  where $k=1,\hdots, d$ are the dimensions of the descriptors describing.
  In our experiments, we used $d = 100$ and the parameters $t_1,\hdots, t_d$
  of HKS and $e_1, \hdots, e_d, \sigma$ of WKS were taken as suggested by the
  respective authors. All descriptors were normalized to have maximum value 1 in
  order to improve their robustness.

  For the 2D case, designing features that can be compared to their 3D
  counterparts can be a challenging task. The difficulty is exacerbated
  here because the shapes are allowed to deform.
  To this end, we consider the solid $U$ of $M$, \ie, $\partial U=M$
  (cf.  Definition~\ref{def:shape}). In other words, we model the 2D
  query as a flat 2-manifold \emph{with} boundary. This new manifold can
  be regarded as a nearly isometric
  transformation (due to flatness and possibly a change in pose) of a
  {\em portion} of the full 3D target (see Fig.~\ref{fig:flatness}).
  Taking this perspective allows us to leverage some recent advances in
  partial 3D matching~\cite{rodola15}, namely that partiality
  transformations of a surface  preserve the Laplacian
  eigenvalues and eigenfunctions, up to some bounded perturbation.

  An implication of this is that we can still compute spectral
  descriptors on the flat solid $U$ and expect them to be
  comparable with those on the full 3D target. By doing so, we make the
  assumption that $U$ can be approximated as a part of
  nearly-isometrically deformed $M$ for the features to be comparable.
  We then define the feature maps
  $f^\text{HKS}_M,f^\text{WKS}_M\colon M\to\IR^d$ on $M$ by
  restricting the descriptors computed on $U$ to its boundary curve
  $\partial U=M$ (see Fig.~\ref{fig:descriptors}, top row).  In other
  words, we have the spectral features $f_M = f_U|_M$ for the
  query shape $M$.

  \begin{figure}[t]
  \centering
  \begin{overpic}
      [trim=0cm 0cm 0cm 0cm,clip,width=0.2\linewidth]{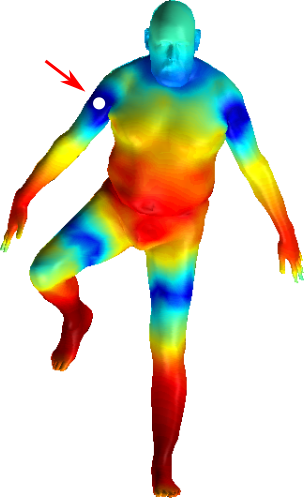}
    \end{overpic}\hspace{5pt}
  \begin{overpic}
      [trim=0cm 0cm 0cm 0cm,clip,width=0.25\linewidth]{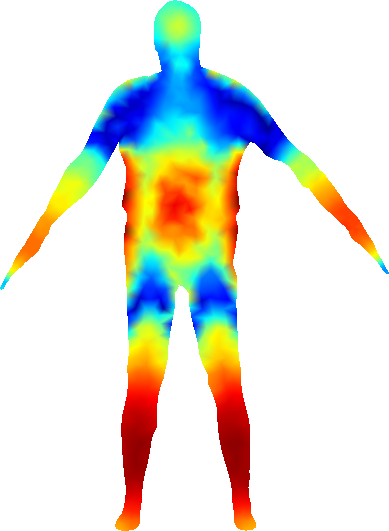}
    \end{overpic}\hspace{5pt}
  \begin{overpic}
      [trim=0cm 0cm 0cm 0cm,clip,width=0.25\linewidth]{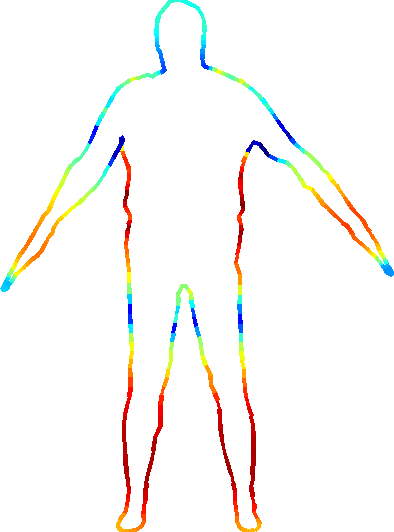}
    \end{overpic}\\
  \begin{overpic}
      [trim=0cm 0cm 0cm 0cm,clip,width=0.2\linewidth]{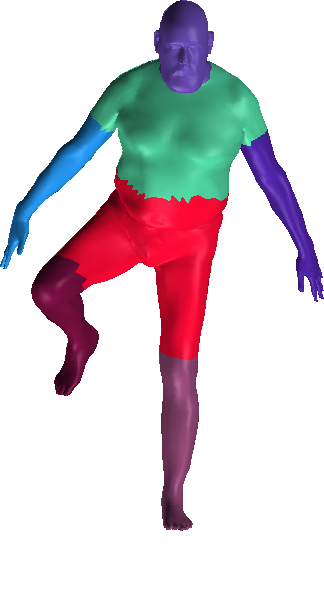}
    \put(12,1){\footnotesize input 3D}
    \end{overpic}\hspace{5pt}
  \begin{overpic}
      [trim=0cm 0cm 0cm 0cm,clip,width=0.25\linewidth]{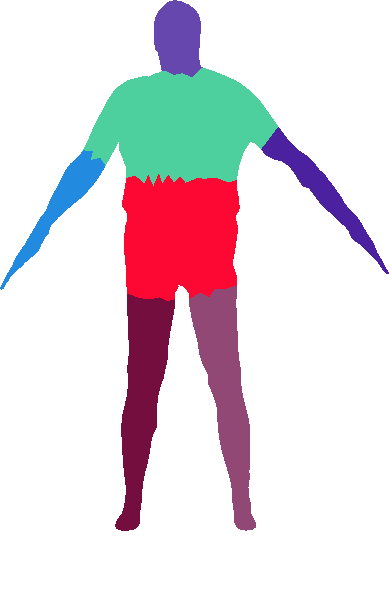}
    \put(4,1){\footnotesize input 2D, tessellated}
    \end{overpic}\hspace{5pt}
  \begin{overpic}
      [trim=0cm 0cm 0cm 0cm,clip,width=0.25\linewidth]{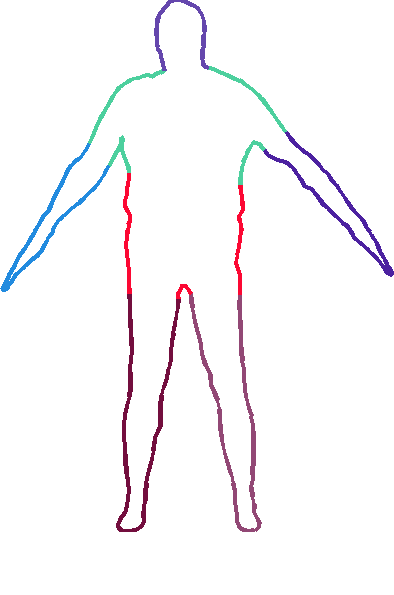}
    \put(20,1){\footnotesize input 2D}
    \end{overpic}
    \caption{\label{fig:descriptors} Computation of local features for
      elastic matching. {\em Top}: $L_1$-distance (blue to red) between
      spectral 3D descriptors of a reference point (white dot) and
      3D descriptors computed on the remaining shape (left) as well as
      2D descriptors on the tessellated query solid (middle). {\em
        Bottom}: Consensus regions detected in 3D and 2D
      using~\cite{rodola-cgf14}. The restricted features (right) are
      used in the energy~\eqref{eq:E_line} to drive the matching
      process. }
  \end{figure}

  \paragraph{Segment features.}
  As an additional `coarse' feature we use corresponding regions on the
  2D query and the 3D shape. Based on the previous observations, we
  are able to automatically extract compatible regions on the two
  objects (namely on $U$ and $N$) by consensus
  segmentation~\cite{rodola-cgf14}, a deformation-invariant region
  detection technique which directly operates with the Laplace-Beltrami
  eigenfunctions of a given shape. The region detection step on the two
  shapes is performed independently; we then obtain the 2D-to-3D region
  mapping by solving a simple linear assignment problem via the
  Hungarian algorithm~\cite{munkres}. Note that this assignment problem
  is typically very small. Assuming we have $r$ regions per shape
  (usually in the range of 5 to 10), the final result of this procedure
  is a pair of corresponding labelings $f_M^\mathrm{SEG} : M\to \mathbb{N}^r$
  and $f_N^\mathrm{SEG} : N\to \mathbb{N}^r$ (see Fig.~\ref{fig:descriptors},
  bottom row).


  \paragraph{Distance function.}

  The final feature maps are obtained by simple concatenation, namely $f
  := (f^\mathrm{HKS}, f^\mathrm{WKS}, f^\mathrm{SEG})$. In order to compare the
  feature maps on $M$ and $N$, we define the distance function:
  \begin{eqnarray}
  \mathrm{dist}(f_M(x) , f_N(y)) &=&
         \|f_M^\mathrm{HKS}(x)- f_N^\mathrm{HKS}(y)\|_1 \\
         &+& \|f_M^\mathrm{WKS}(x)- f_N^\mathrm{WKS}(y)\|_1 \nonumber 
  \end{eqnarray}
  if ${f_M^\text{SEG}(x)=f_N^\text{SEG}(y)}$, and set
  ${\dist(f_M(x),f_N(y))=\tau}$ otherwise.  Here, $\tau>0$ is a
  positive value \zorah{larger than any other values in the feature matrix} to prevent matching points belonging to different regions. In our experiments, we used $\tau=10^3$.


  %
  \begin{figure}[t]
  \centering
%
%
\definecolor{mycolor1}{rgb}{0.00000,0.70000,1.00000}%
\definecolor{mycolor2}{rgb}{0.00000,1.00000,1.00000}%
\definecolor{mycolor3}{rgb}{1.00000,1.00000,0.00000}%
\begin{tikzpicture}

\begin{axis}[%
width=0.5\linewidth,
at={(0.758in,0.481in)},
scale only axis,
every outer x axis line/.append style={black},
every x tick label/.append style={font=\color{black}, font=\footnotesize},
xmin=0,
xmax=0.25,
xlabel near ticks,
ylabel near ticks,
xlabel={\footnotesize{Geodesic error}},
xmajorgrids,
every outer y axis line/.append style={black},
every y tick label/.append style={font=\color{black}, font=\footnotesize},
ymin=0,
ymax=100,
ylabel={\footnotesize{\% Matches}},
ymajorgrids,
axis background/.style={fill=white},
xtick={0,0.05,0.1,0.15,0.2,0.25},
xticklabels={0,0.05,0.1,0.15,0.2,0.25},
axis x line*=bottom,
axis y line*=left,
legend style={at={(0.97,0.03)},anchor=south east,legend cell align=left,align=left,draw=black}
]
\addplot [color=blue,solid,line width=2.0pt,mark options={solid}]
  table[row sep=crcr]{%
0	0.995845552297165\\
0.01	3.98194436486564\\
0.02	9.83940891292187\\
0.03	19.2142027599399\\
0.04	27.2372906476953\\
0.05	37.9218541574033\\
0.06	52.8397806432489\\
0.07	63.7147991700015\\
0.08	70.2869106699752\\
0.09	76.4650218859194\\
0.1	82.9753040032791\\
0.11	86.8375119551851\\
0.12	89.316163410302\\
0.13	91.2692990845744\\
0.14	92.5640456346495\\
0.15	93.8367604864052\\
0.16	94.8430454980189\\
0.17	95.7605205629184\\
0.18	96.589185681104\\
0.19	97.4178507992895\\
0.2	98.5276335564968\\
0.21	98.9863710889466\\
0.22	99.3562986746823\\
0.23	99.7262262604181\\
0.24	99.911190053286\\
0.25	100\\
};
\addlegendentry{\footnotesize 25};

\addplot [color=mycolor1,solid,line width=2.0pt,mark options={solid}]
  table[row sep=crcr]{%
0	1.07785881529725\\
0.01	2.61737210974646\\
0.02	8.49938225345837\\
0.03	15.8258562372147\\
0.04	26.1478945982622\\
0.05	39.5708550858478\\
0.06	51.3134709254046\\
0.07	59.1311040447208\\
0.08	64.2332853989022\\
0.09	67.5429890831021\\
0.1	71.2602029449294\\
0.11	74.6729994126344\\
0.12	79.1070018026047\\
0.13	82.205658963401\\
0.14	83.8065339355518\\
0.15	85.7312701273976\\
0.16	86.9440787474936\\
0.17	88.2599801510947\\
0.18	89.8802989488182\\
0.19	90.990014785409\\
0.2	92.8459932564139\\
0.21	94.6095525165418\\
0.22	95.9797999992311\\
0.23	96.5618860510805\\
0.24	96.9548133595285\\
0.25	97.9916148502218\\
};
\addlegendentry{\footnotesize 50};

\addplot [color=mycolor2,solid,line width=2.0pt,mark options={solid}]
  table[row sep=crcr]{%
0	0\\
0.01	0.49350443599493\\
0.02	2.52028374750671\\
0.03	9.76956762998845\\
0.04	20.6182250069864\\
0.05	32.9179757232887\\
0.06	43.9830959397512\\
0.07	51.5700720366732\\
0.08	57.1449328749181\\
0.09	60.7940406024885\\
0.1	65.8361984282908\\
0.11	70.780124426981\\
0.12	75.7003110674525\\
0.13	79.3138097576948\\
0.14	82.3023084479371\\
0.15	84.534626719057\\
0.16	86.6627783235102\\
0.17	87.9813359528487\\
0.18	89.20759659463\\
0.19	90.6303208906352\\
0.2	92.3596103470858\\
0.21	93.8852107157864\\
0.22	95.3340828478523\\
0.23	96.9737404938468\\
0.24	97.5612387430869\\
0.25	98.6849710982659\\
};
\addlegendentry{\footnotesize 75};

\addplot [color=white!50!green,solid,line width=2.0pt,mark options={solid}]
  table[row sep=crcr]{%
0	0\\
0.01	0.388349514563107\\
0.02	2.9558056099152\\
0.03	7.96912372254838\\
0.04	19.1834665815067\\
0.05	32.3509225241229\\
0.06	41.8134471608572\\
0.07	47.3676773386087\\
0.08	53.2995790192286\\
0.09	58.4687971401202\\
0.1	64.6032981109084\\
0.11	69.2959910972152\\
0.12	73.2838860352705\\
0.13	77.0840351720562\\
0.14	79.1493399006823\\
0.15	81.1663707903871\\
0.16	83.5943857949104\\
0.17	85.7288587250697\\
0.18	87.6699080575501\\
0.19	89.70766918887\\
0.2	92.0367012759901\\
0.21	93.5897678538308\\
0.22	95.1394277538259\\
0.23	96.7846638933786\\
0.24	97.3706124828247\\
0.25	97.8553070363038\\
};
\addlegendentry{\footnotesize 100};


\addplot [color=red,solid,line width=2.0pt,mark options={solid}]
  table[row sep=crcr]{%
0	0.379916753011383\\
0.01	1.33974502348468\\
0.02	6.32708949501588\\
0.03	12.6996197718631\\
0.04	20.5274733467531\\
0.05	31.266681577574\\
0.06	41.287184075151\\
0.07	48.3155781997839\\
0.08	53.6201173382739\\
0.09	59.470123414831\\
0.1	64.7498695295609\\
0.11	68.3146946991724\\
0.12	72.0815626630881\\
0.13	76.2162266481396\\
0.14	79.0419742041303\\
0.15	81.3254593175853\\
0.16	83.7826154083681\\
0.17	86.3381966960012\\
0.18	88.5004631773969\\
0.19	90.7761127264594\\
0.2	92.8792962051741\\
0.21	94.9854618653545\\
0.22	96.2301498546186\\
0.23	97.2946169276369\\
0.24	97.9278273283673\\
0.25	98.2761134193454\\
};
\addlegendentry{\footnotesize 200};

\end{axis}
\end{tikzpicture}%
  \includegraphics[width=0.27\linewidth]{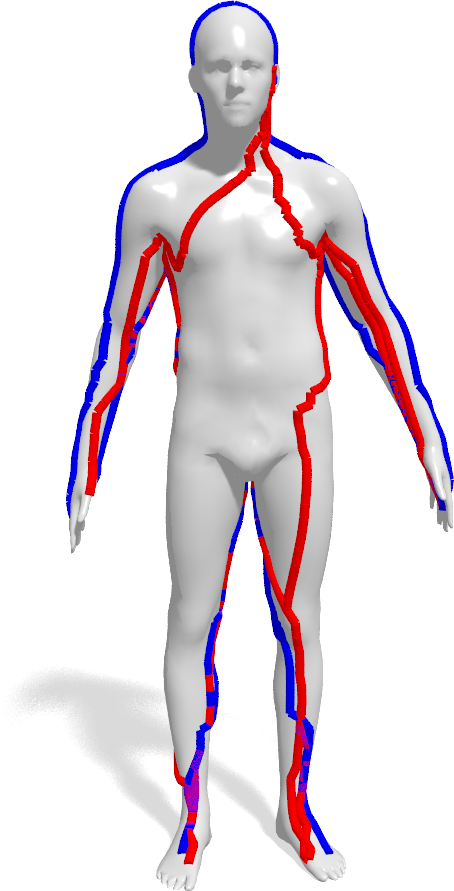}
  \caption{\label{fig:sensitivity} Sensitivity of the spectral features
    to increasing number of eigenfunctions. On the right we show typical
    solutions obtained when using 25 (blue) and 200 (red)
    eigenfunctions. \zorah{See~Eq.~\eqref{eq:err} for details on the error curve.}}
  \end{figure}

  \subsection{Sensitivity analysis}
  In most shape analysis applications, only the first $k$ eigenfunctions of
  $\Delta$ are used to define $f^\mathrm{WKS}$ and $f^\mathrm{HKS}$. In the classical
  3D-to-3D setting, for large $k$ the resulting descriptors tend to be
  more accurate, but at the same time become more sensitive to the lack
  of isometry relating the two shapes.
  We performed a sensitivity analysis of our elastic matching method on
  a subset of our FAUST-derived dataset to determine the optimal number of
  eigenfunctions.

  We observed a similar trend to the 3D-to-3D case in our 2D-to-3D setting where
  naturally the isometry is only approximate giving better results for a low $k$,
   as reported in Fig.~\ref{fig:sensitivity}.
  From this analysis we selected $k=25$ as
  the fixed number of eigenfunctions for all subsequent experiments. In
  the figure, we plot cumulative curves showing the percent of matches
  with a geodesic error smaller than a variable
  threshold (see~Eq.~\eqref{eq:err}).

  \subsection{Runtime}\label{sec:runtime}

  We implemented our method in C++ and ran it on an Intel Core i7
  3.4GHz CPU.  In Fig.~\ref{fig:runtimes} we show the execution times of
  our method on the FAUST dataset (10 queries and 100 targets). The
  plotted results show that for 3D shapes of fixed size, in practice the
  runtime grows linearly with the number of 2D query points $m$ and
  quadratic with the number of points on the 3D target $n$.

  Fig.~\ref{fig:eps_runtimes} shows the decrease of the average runtime with growing
  parameter $\epsilon$ for the approximation algorithm. The plot shows that
  after decreasing rapidly over $\epsilon \approx 0.05$ the runtime stays
  relatively constant with no chang in the MAP. Therefore, choosing a small
  $\epsilon$ leads to a good quality-runtime ratio.

  \begin{figure}[t]
  \centering
  \begin{overpic}
    [trim=2cm 1cm 4.5cm 1cm,clip,width=0.45\linewidth]{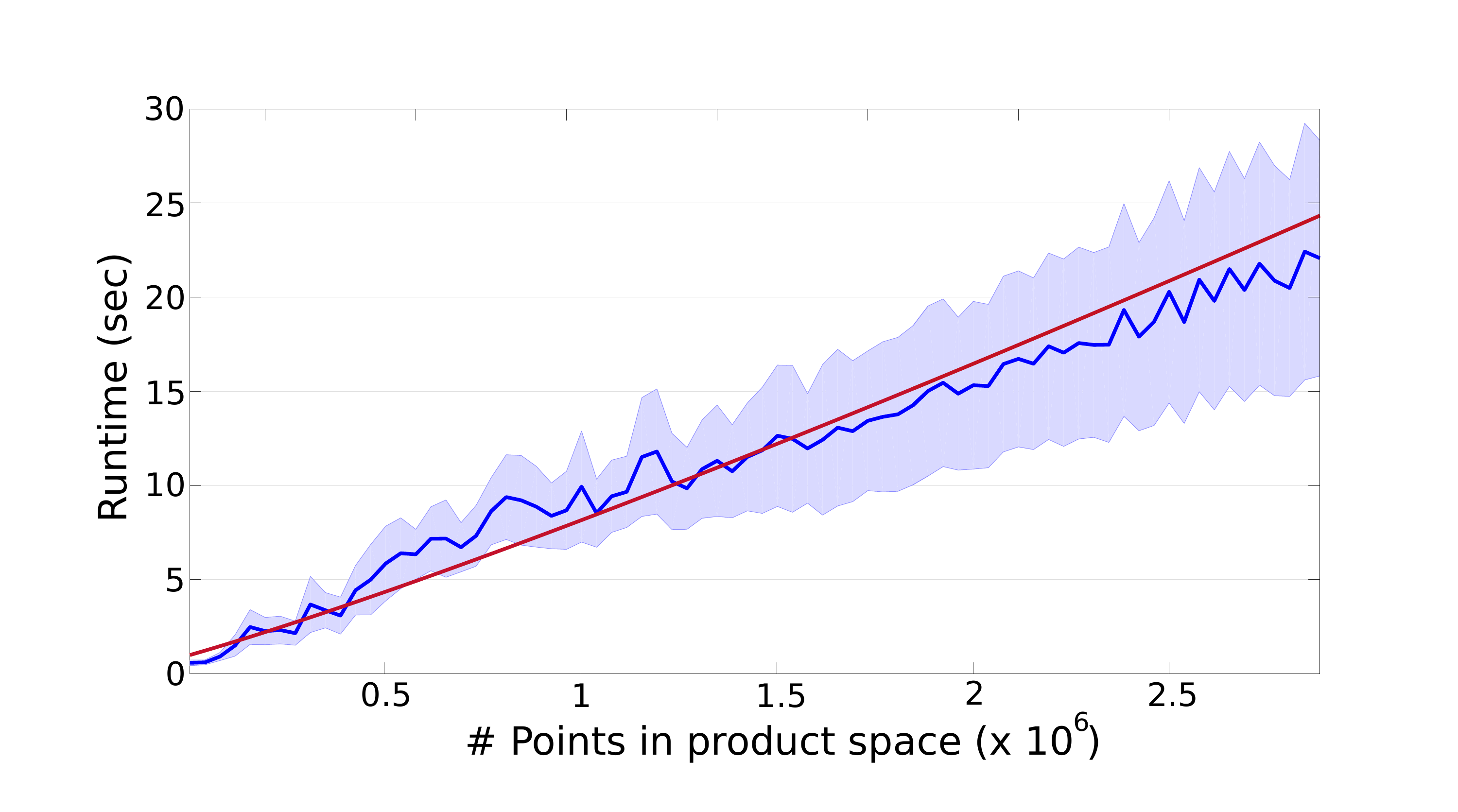}
  \end{overpic}
    \begin{overpic}
      [trim=2cm 1cm 4.5cm 1cm,clip,width=0.45\linewidth]{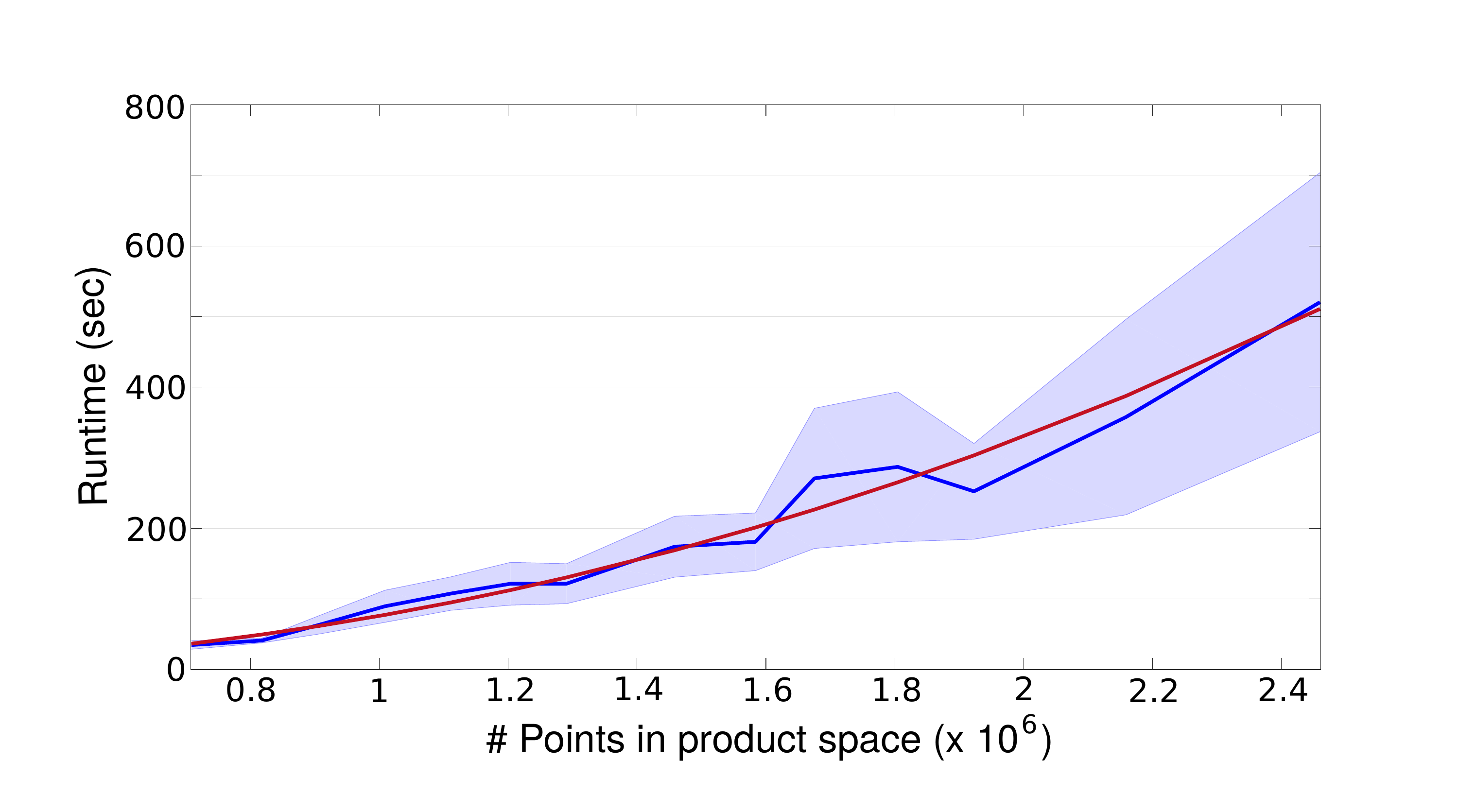}
    \end{overpic}
    \caption{\label{fig:runtimes}Blue: Mean runtime with standard deviation shaded.
     Red: Optimal curve of form $a \cdot x^b$ fitted to the data.
     Left: Runtime of our matching method on 3D targets with $\sim$7k vertices (FAUST dataset).
    On the $x$ axis we vary the size of the 2D query from 25 to 400 points. The
    fitted curve has $a = 0.0245$ and $b = 1.1518$ so the real runtime is nearly
    linear in $m$.
    Right: Runtime with a $100$ vertex query and upsampled FAUST shapes. On the
    $x$ axis we vary the size of the 3D targets from 7k to 24k. The fitted curve
    has $a = 2.6466 \cdot 10^{-7}$ and $b = 2.1147$ making the real runtime
    nearly quadratic in $n$ with a really small constant.}
  \end{figure}


  \subsection{Sketch-based shape retrieval}\label{sec:exp-retrieval}

  We retrieve the most fitting element from a collection of 3D shapes in comparison to the query 2D silhouette.
  The matching energy $E(\phi^*)$ obtained by our optimization problem is a suitable measurement for shape similarity and we match the query to every instance of the collection, choosing the one with lowest energy as the result.

  We evaluated the performance of our retrieval pipeline on the extended
  2D-to-3D TOSCA dataset.
  As baselines for our comparisons we use the spectral retrieval method
  of \cite{reuter06} and a pure region-based retrieval technique using
  segments computed with \cite{rodola-cgf14}; both are the foundation
  for the features we use. The rationale of these
  experiments is to show that these features are not
  sufficient to guarantee good retrieval performance. However, using
  these quantities in our elastic matching pipeline enables promising
  results even in challenging cases.
  The first baseline method we compare against is Shape-DNA
  \cite{reuter06} using the (truncated) spectrum of the Laplace-Beltrami
  operator as a global isometry-invariant shape
  descriptor. We apply this method to compare targets in the 3D database
  with flat tessellations of the 2D queries.
  The second method used in the comparisons is a simple evaluation of
  the matching cost obtained when putting the consensus regions into
  correspondence via linear assignment (see Fig.~\ref{fig:descriptors}).
  Since this step typically produces good coarse 2D-to-3D matchings, it
  can be used as a retrieval procedure per se.

  \begin{figure}[b]
  \centering
  \begin{overpic}
    [trim=3.8cm 0.5cm 1.7cm 1.6cm,clip,width=0.5\linewidth]{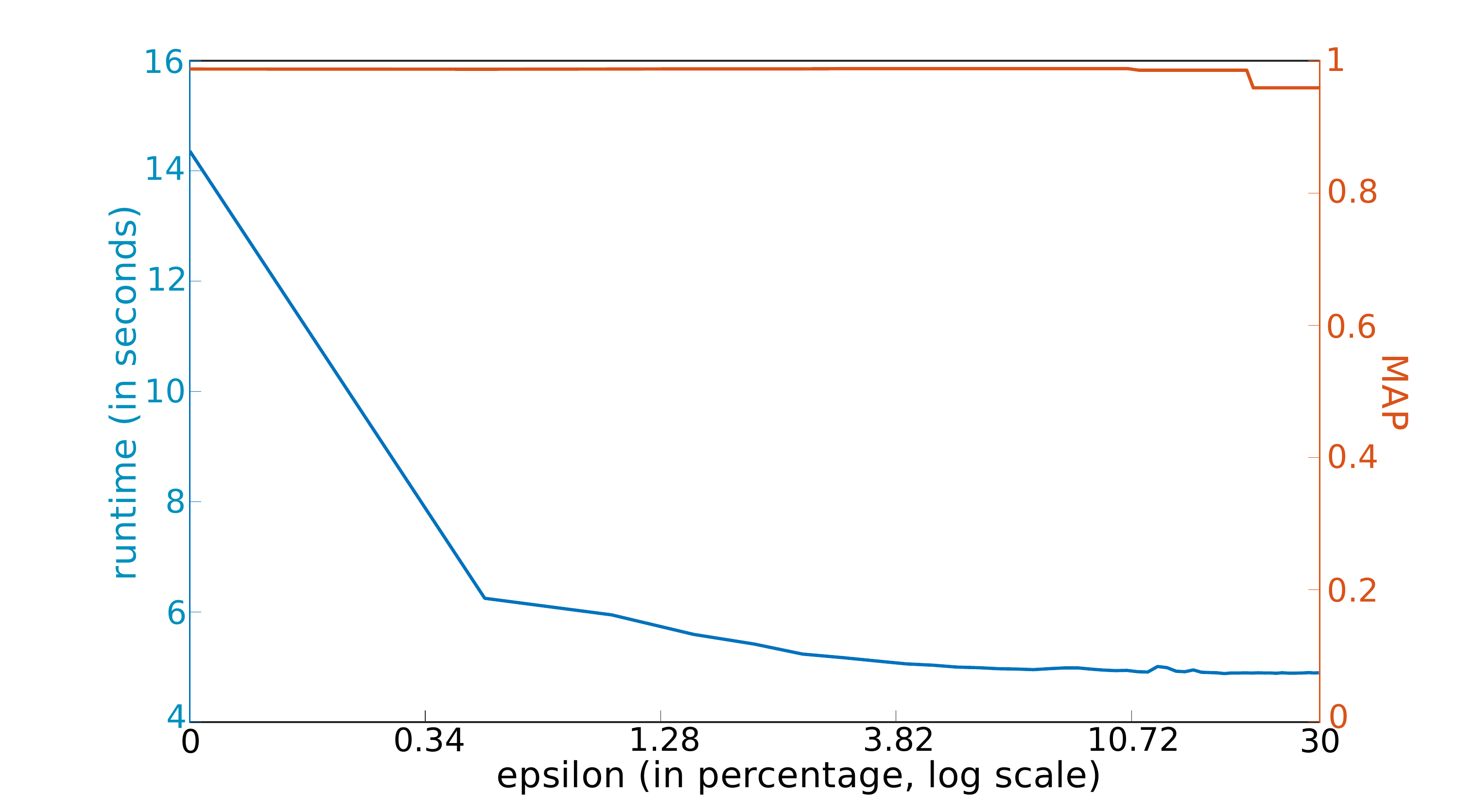}
  \end{overpic}
    \caption{\label{fig:eps_runtimes} The average runtime of the TOSCA dataset with
    values between $0$ (globally optimal) and $0.3$ (within 30\% of the optimum).
    The runtime nearly stagnates after $\epsilon = 0.05$ (only one or two iterations)
     but small $\epsilon$ already lead to significant improvement in the runtime
     and comparable results (as shown in Table \ref{tab:results}). The segmentation
     feature was not used. }
  \end{figure}

  We test the global optimization using all features we presented before. For
  the approximate optimization we use all features except the segment features.
  The reason is we noticed that the energies without the segments are
  comparable for same-class matching; the only difference being the matchings
  only take place on one symmetric half of the 3D shape (see Fig.~\ref{fig:symmetry}).
  Although this is not the favored solution, since it does not substantially
  change the energy, it is suitable for retrieval. Unfortunately, without the
  coarse matching of the segments the globally optimal solution for interclass
  matchings can take longer to converge because there are many equally best solutions -
  still with a high energy in comparison. Therefore, we use the segment features
  for the globally optimal method but show that we can get similar results
  without this feature and the approximate method.

  The results of the shape retrieval experiments are reported in
  Table~\ref{tab:results}. $\epsilon = 0.001$ was used for the approximation
  experiments. We used average precision (AP) and mean
  average precision (MAP) as measures of retrieval
  performance\footnote{\emph{Precision} measures the
  percentage of correctly retrieved shapes.}. The results slightly differ from
  the ones reported in \cite{LRSBC16} due to a small bug in the code.
  Additional qualitative
  examples of solutions obtained with our method are shown in
  Fig.~\ref{fig:rankings}. Fig.~\ref{fig:symmetry} shows an embedding
  and clustering of all 3D models using the matching energies (with
  the approximation algorithm).

  \begin{figure}[t]
  \centering
      \includegraphics[trim=0cm 0cm 0cm 0cm,clip,width=0.25\linewidth]{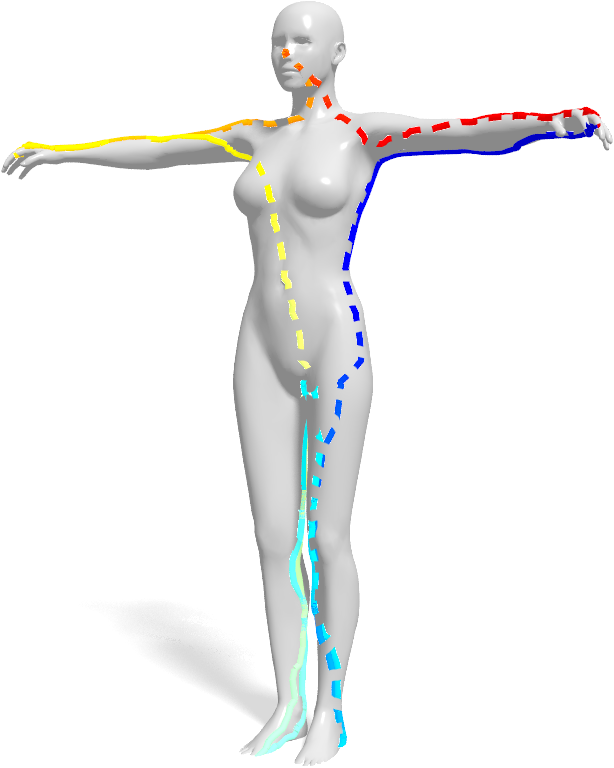}
      \includegraphics[trim=0cm 0cm 0cm 0cm,clip,width=0.25\linewidth]{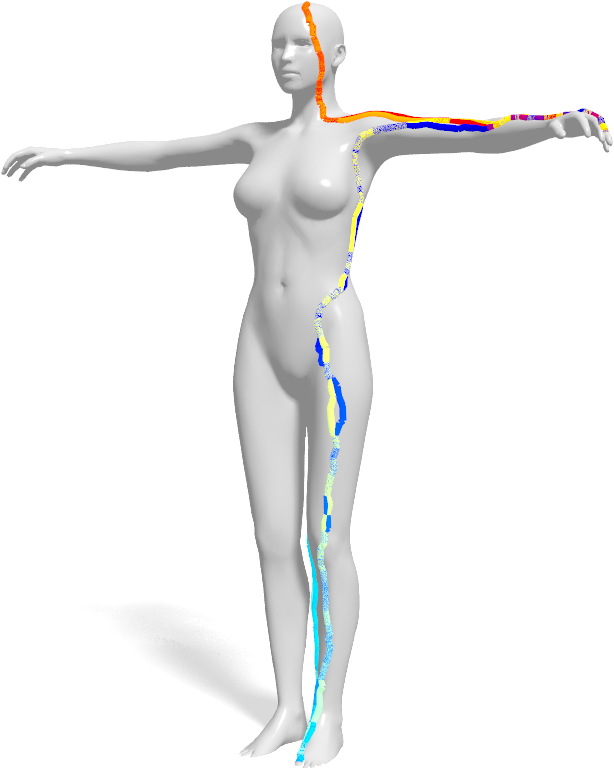}
%
%
\definecolor{mycolor1}{rgb}{1.00000,0.60000,0.00000}%
\definecolor{mycolor2}{rgb}{0.80000,1.00000,0.00000}%
\definecolor{mycolor3}{rgb}{0.00000,1.00000,0.40000}%
\definecolor{mycolor4}{rgb}{0.00000,1.00000,1.00000}%
\definecolor{mycolor5}{rgb}{0.20000,0.00000,1.00000}%
\definecolor{mycolor6}{rgb}{0.80000,0.00000,1.00000}%
\definecolor{mycolor7}{rgb}{1.00000,0.00000,0.60000}%
\begin{tikzpicture}[scale=0.75]

\clip (0.5,0) rectangle (7.3,4.9);

\begin{axis}[%
hide axis,
width=3.1in,
height=1.811in,
scale only axis,
scale=1.4,
xmin=-30,
xmax=30,
ymin=-20,
ymax=25,
axis background/.style={fill=white},
axis x line*=bottom,
axis y line*=left,
]
\addplot[only marks,mark=*,mark options={},mark size=0.7906pt,color=red] plot table[row sep=crcr]{%
-25.7061870130095	-8.32316520745946\\
-22.4983234804446	-9.78638674996485\\
-25.6331295002135	-9.58358759898243\\
-24.2567681061133	-9.40039201466008\\
-22.8059776645963	-9.70333240294033\\
-25.4770266984221	-8.86048081327069\\
-25.2825494201239	-9.35498594726497\\
-24.6459750470513	-10.2070928289784\\
-25.8517391950952	-9.51823627405765\\
-24.9874066511289	-8.8729297417764\\
-26.0692448832202	-8.99245819173323\\
};
\draw[rotate around={0:(136,41)},red,fill=red!10] (57,107) ellipse (12pt and 7pt);
\node at (57.5,133) {cat};

\addplot[only marks,mark=*,mark options={},mark size=0.7906pt,color=mycolor1] plot table[row sep=crcr]{%
-15.9138500960859	-16.5160236906396\\
-16.4439948073962	-15.6392589953414\\
-16.3793615087315	-16.2983581237747\\
-16.3628457540301	-15.5093515234943\\
-16.0983316401937	-16.4822898230285\\
-16.483205663092	-15.2183764106625\\
};
\draw[rotate around={20:(136,41)},mycolor1,fill=mycolor1!10] (136,41) ellipse (3pt and 5.5pt);
\node at (146,20) {centaur};

\addplot[only marks,mark=square*,mark options={},mark size=0.7906pt,color=red, shift={(-200,-100)}] plot table[row sep=crcr]{%
21.9923783655238	16.0769252687003\\
21.2795805776282	15.6667487325097\\
18.1273963281802	19.4021781999104\\
20.016062415072	18.1706458082932\\
24.4585309169136	16.1143387687076\\
21.7246619183468	17.5236474528848\\
20.2621189894171	17.7216594763838\\
21.3868823631063	13.4832010726468\\
20.7708750182469	15.0171801122232\\
21.6552452405132	13.9706688921056\\
20.0504385989202	14.6039978473483\\
18.5021317530851	21.9109653229905\\
21.7097292687255	13.1910109486716\\
20.0345935570551	12.9250047120366\\
};
\draw[rotate around={-40:(328,257)},black,fill=black!10] (327,254) ellipse (40pt and 26pt);
\node at (347,165) {humans};

\addplot[only marks,mark=*,mark options={},mark size=0.7906pt,color=green!75!mycolor2] plot table[row sep=crcr]{%
-17.391222100494	-3.74606641706863\\
-19.2229224133857	-4.16037218071808\\
-15.5726456654873	-5.11470715157898\\
-16.7308069329988	-4.18475763594098\\
-16.2133325466878	-4.91684962369828\\
-15.7036387307449	-4.82495659601147\\
-19.1197698919833	-4.16023788646538\\
-19.1454310831552	-4.13405218799834\\
-17.5484179833855	-4.19537996636259\\
-19.0093149728513	-5.30393427787954\\
-18.6619045902998	-5.81985454573828\\
};
\draw[rotate around={0:(162,150.5)},green!75!mycolor2,fill=green!75!mycolor2!10] (125,153) ellipse (12pt and 7pt);
\node at (126,179) {dog};

\addplot[only marks,mark=square*,mark options={},mark size=0.7906pt,color=mycolor3] plot table[row sep=crcr]{%
-15.0561808795033	-16.6155275024351\\
-11.8537048242601	-5.18502433676305\\
-10.7100988049221	-4.24681127293918\\
-10.8359493569048	-3.76270407361937\\
-11.8998702434312	-4.46899459768938\\
-10.5272546302893	-4.69955918862084\\
-13.2441239744891	-4.42424477816158\\
-10.2404297619809	-3.71535621410944\\
-9.43363436203128	-4.39785709243563\\
-9.54102074246242	-4.01853293814299\\
-9.93849361013639	-4.26628875986622\\
-10.6223010785298	-6.06072460875534\\
-11.8788847889804	-5.24528506976664\\
-9.95999270245071	-5.72722709148134\\
-9.91064336731434	-6.13594945522733\\
-9.56175114264456	-5.11556890495404\\
-10.487308391681	-5.59979076958603\\
};
\draw[rotate around={0:(162,150.5)},mycolor3,fill=mycolor3!10] (188.5,150.5) ellipse (13pt and 8pt);
\node at (189,182) {horse};

\addplot[only marks,mark=*,mark options={},mark size=0.7906pt,color=mycolor4] plot table[row sep=crcr]{%
-18.1533177891658	-11.4102556696892\\
-18.8342294654198	-9.80013921428249\\
-18.1868990189761	-8.52880839397114\\
-18.5658390644973	-11.2311522131935\\
-18.2543502093681	-12.5098821515779\\
-19.1811502657502	-9.10045759749459\\
-17.6853006240631	-8.42912318506582\\
-17.3807260340387	-9.18933575229723\\
-18.3299175386317	-8.16553999244086\\
-18.6409641262602	-11.5484307394761\\
-17.3700212289653	-9.50743674822187\\
-18.741904658445	-10.5432846162438\\
-18.7151153507216	-7.29237591892657\\
-18.302942709759	-8.71919221737122\\
-18.000472030358	-9.73872246722523\\
};
\draw[rotate around={15:(117,100)},mycolor4,,fill=mycolor4!10] (117,100) ellipse (7pt and 15pt);
\node at (155,77) {lioness};

\addplot[only marks,mark=*,mark options={},mark size=0.7906pt,color=blue!60!mycolor4, shift={(-200,-100)}] plot table[row sep=crcr]{%
22.3724816231077	14.4559354878952\\
23.5381906540329	15.9345866123393\\
17.727033261104	21.4365600217525\\
23.0581996330542	16.4348726253236\\
21.7489816058816	15.4657762773053\\
18.220168560663	20.596584804794\\
21.5230392061457	17.8473901541331\\
19.7186546469138	18.9266616850045\\
18.3569320731468	18.717756879425\\
17.3335659468411	21.5442656934547\\
19.6937301255121	17.6847707049863\\
17.9312824485019	21.8883493038351\\
26.4343839524184	10.9889031083592\\
21.9258070153756	12.6716999719948\\
17.4251283319948	21.159940097627\\
22.2346355891473	15.4165897634195\\
20.3387126995114	16.6895509173931\\
17.4353008108083	22.0795169460823\\
19.0698037977515	18.0575802564327\\
24.6219553533097	15.8046162059868\\
};

\addplot[only marks,mark=*,mark options={},mark size=0.7906pt,color=mycolor5] plot table[row sep=crcr]{%
-19.0753454197408	-17.4113905139107\\
-19.8135394637465	-16.3218702655492\\
-19.9872904770246	-17.0085175705673\\
-19.2159130418159	-17.1301083218741\\
-19.4935144649366	-16.8829335266588\\
-19.5916269256187	-17.7568097202545\\
};
\draw[rotate around={-40:(104,30)},mycolor5,,fill=mycolor5!10] (104,30) ellipse (6pt and 5pt);
\node at (63,30) {seahorse};

\addplot[only marks,mark=*,mark options={},mark size=0.7906pt,color=mycolor6, shift={(-200,-100)}] plot table[row sep=crcr]{%
27.3706895794483	12.9060745520793\\
27.7902962631205	13.5469452840439\\
28.9966996178492	13.2704127487051\\
28.068983049843	13.5261138258908\\
25.9462221203223	16.0711942298857\\
27.0814691510646	11.4573623330214\\
27.5393165055076	14.0424586820986\\
28.2868385053046	14.1085716214832\\
18.3173801319617	17.4715252508638\\
28.5728442944805	12.8698233029549\\
27.8870281873644	14.5269008693289\\
27.8965303079733	12.8727674713243\\
23.5845880812162	9.10647472741062\\
21.6691535487056	9.06106080227708\\
23.927842680843	9.47163904056614\\
21.6206286262817	10.2230137759189\\
25.1799168840996	10.1431077213116\\
22.6969331787462	10.2167339924814\\
22.3968033175765	10.0700877164242\\
22.6815550564758	9.26311104959729\\
22.407207265268	9.39580581804513\\
23.2921854569327	9.64355094687255\\
19.5910417576198	22.2762636503647\\
22.6969331787459	10.2167339924818\\
};

\addplot[only marks,mark=*,mark options={},mark size=0.7906pt,color=mycolor7] plot table[row sep=crcr]{%
-14.2391032816322	-9.34146775466156\\
-14.2730801542628	-9.77540757809286\\
-14.7286881397652	-9.73120755605578\\
};
\draw[rotate around={0:(-14,-9)},mycolor7,,fill=mycolor7!10] (155,105) ellipse (4pt and 4pt);
\node at (180,100) {wolf};

\end{axis}
\end{tikzpicture}%
    \caption{\label{fig:symmetry} (Left and Middle) Example of the same query matched to the same
      shape, once using the segmentation features (left) and once without (right).
      Dashed lines indicate a path on the backside normally not visible in this
      perspective. The matchings are very similar but without segments the path
      only considers one of the symmetric sides (normally because it is slightly shorter). (Right) 2-dimensional embedding of all 3D models using the
      approximate matching energies to each of the $k$ sketches as $k$-dimensional
      coordinates for each model. Except for one outlier (horse close to the
      centaurs), all classes are perfectly separable. }
  \end{figure}

  \begin{table}[t]
    \centering
    {\renewcommand{\arraystretch}{1.0}
      \begin{tabular}{lcccc}
        \hline
        &\textbf{Global (Ours)}& \textbf{Approximation (Ours)}& Shape DNA & Consensus Segmentation \\
        \hline
        {\em cat} &\textbf{1.0000}&\textbf{1.0000}&0.2852&0.2518\\
        {\em dog} &\textbf{1.0000}&0.9615&0.3519&0.3970\\
        {\em horse} &\textbf{0.9987}&0.9313&0.4469&0.3432\\
        {\em human} &\textbf{1.0000}&\textbf{1.0000}&0.7447&0.9985\\
        {\em lioness} & \textbf{0.9984}&0.9587&0.7022&0.5419 \\
        {\em seahorse} & \textbf{1.0000}&\textbf{1.0000}&0.0779&\textbf{1.0000} \\
        {\em wolf} &\textbf{1.0000}&0.8082&0.2230&0.2470\\
        {\em human (hd)} & \textbf{1.0000}&\textbf{1.0000}&0.7096&0.9462 \\
        {\em cat (hd)} & \textbf{0.9888}&0.9772&0.8622& - \\
        \hline
        MAP &\textbf{0.9984}&0.9597&0.4720&0.5907\\
        \hline
      \end{tabular}}
    \vspace{2mm}
    \caption{\label{tab:results}
      Retrieval results on the 2D-to-3D extended dataset.
      For each method we show per-class AP and, in the last row, the MAP.
      The global algorithm uses all features and the approximation does not
      use the segmentation. $\epsilon$ was chosen to be $0.001$. hd refers
      to handdrawn sketches. Shape DNA refers to \cite{reuter06} and Consensus Segmentation to
      \cite{rodola-cgf14}.
    }
  \end{table}

  \begin{figure*}[t]
  \centering
  \includegraphics[width=0.19\linewidth]{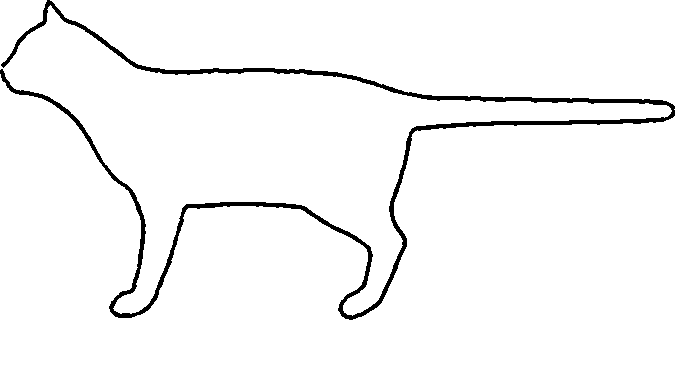}\hspace{5pt}
  \includegraphics[width=0.15\linewidth]{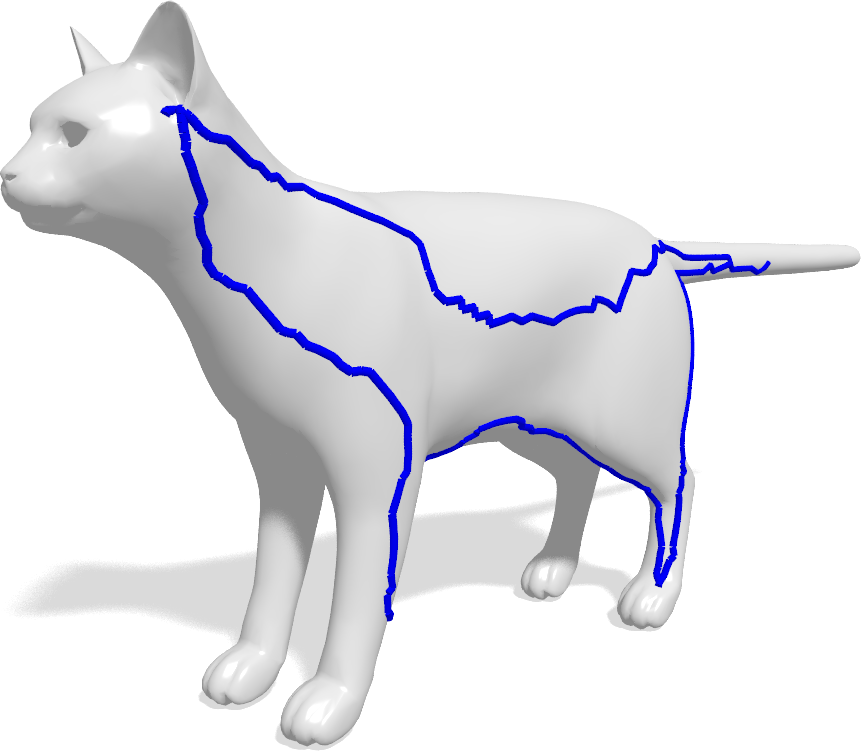}
  \includegraphics[width=0.13\linewidth]{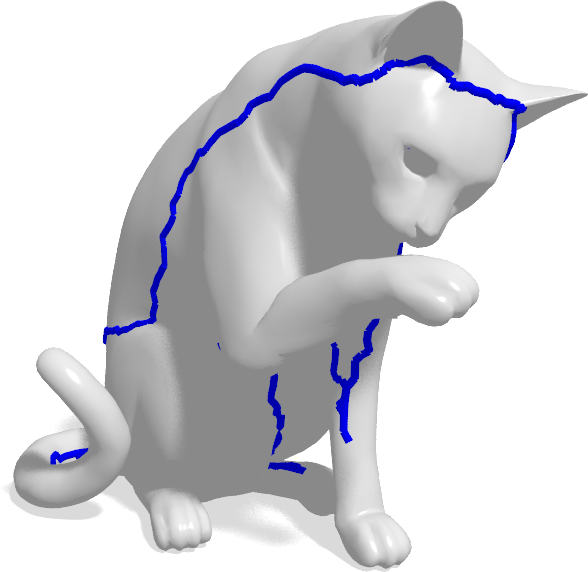}
  \includegraphics[width=0.15\linewidth]{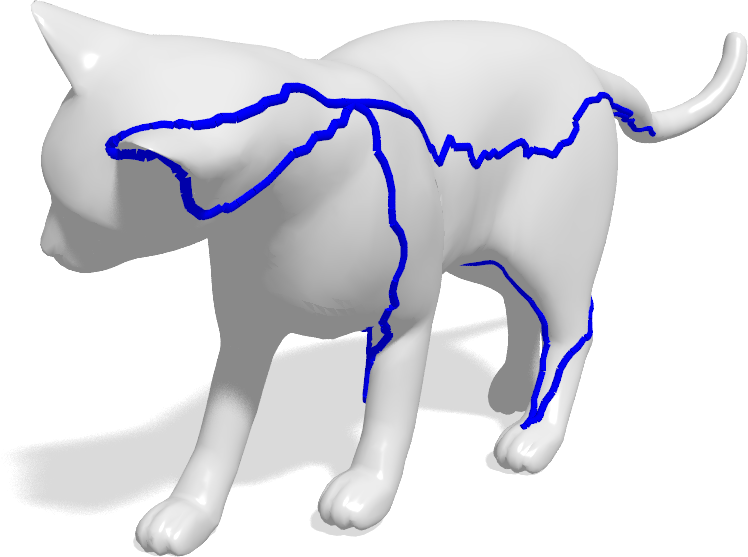}
  \includegraphics[width=0.16\linewidth]{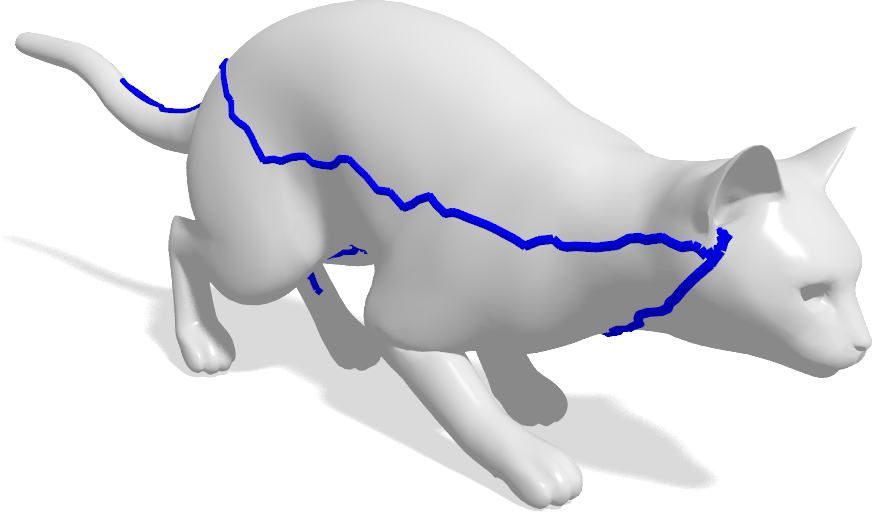}
  \includegraphics[width=0.16\linewidth]{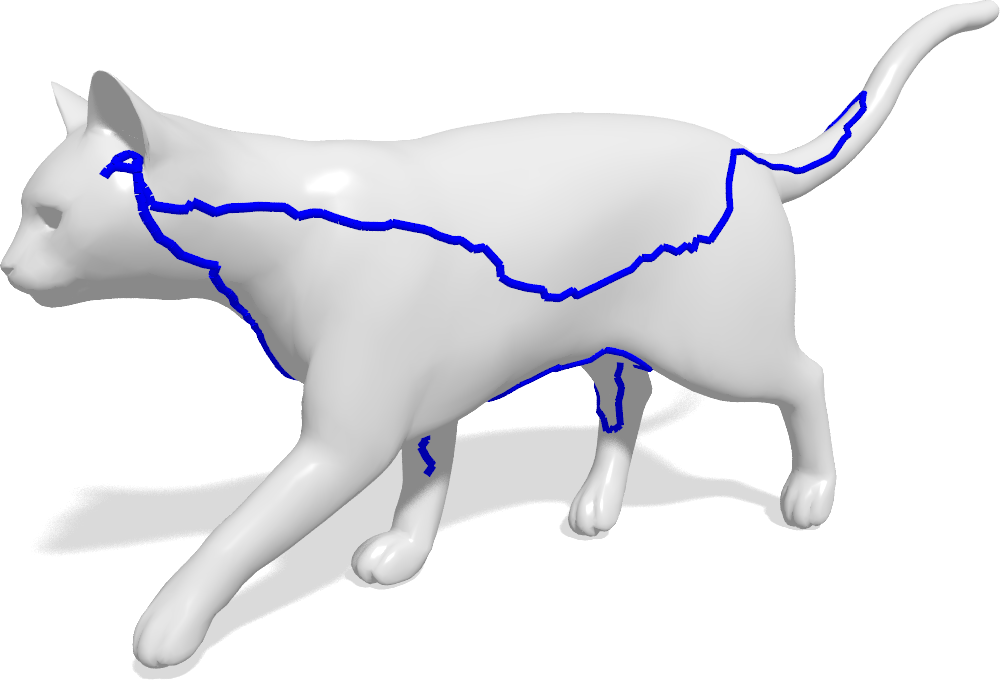}\\
  \includegraphics[width=0.18\linewidth]{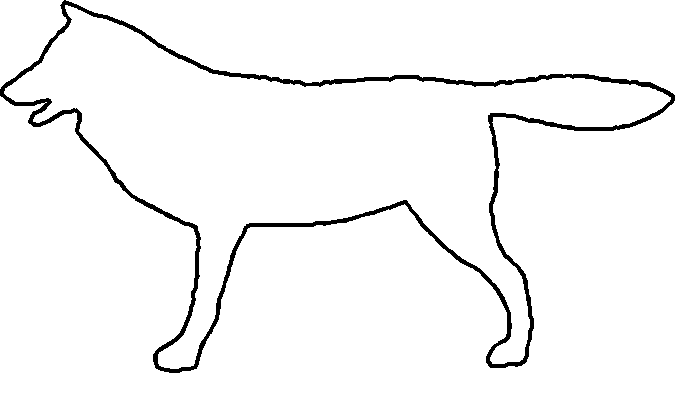}\hspace{5pt}
  \includegraphics[width=0.15\linewidth]{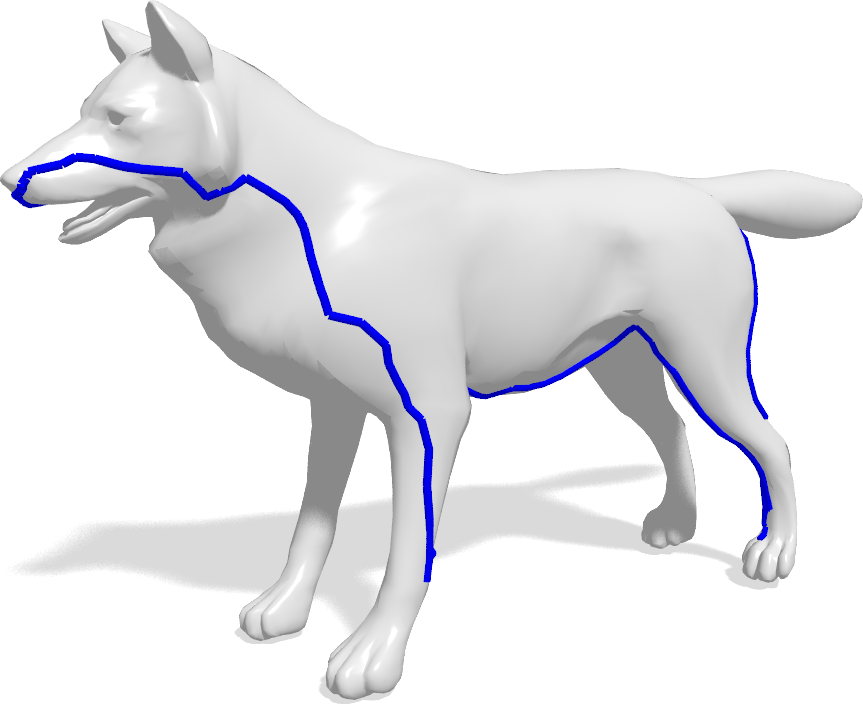}
  \includegraphics[width=0.13\linewidth]{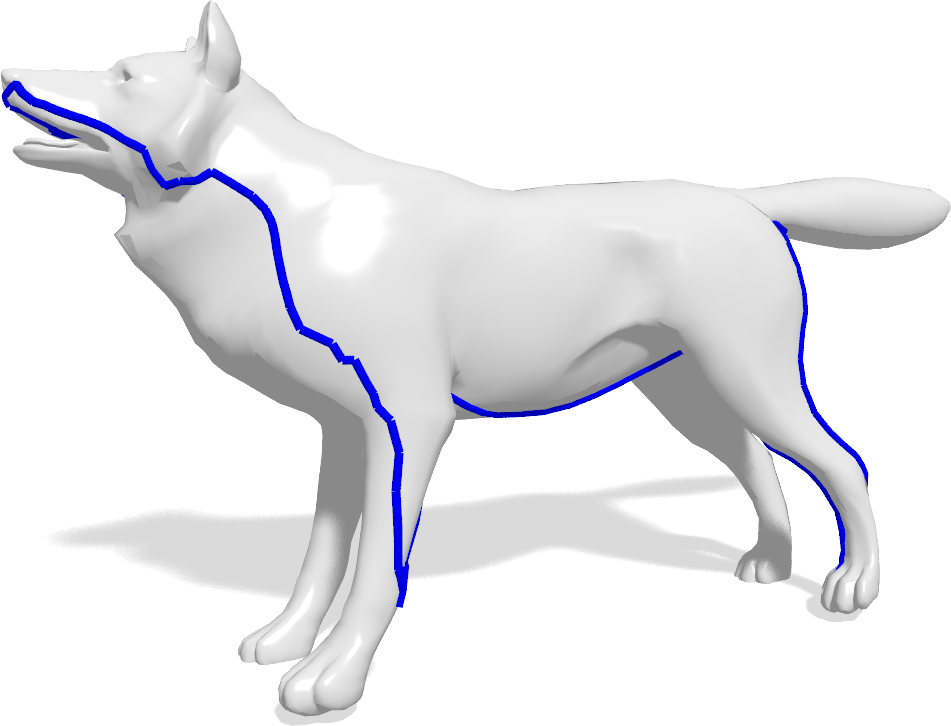}
  \includegraphics[width=0.15\linewidth]{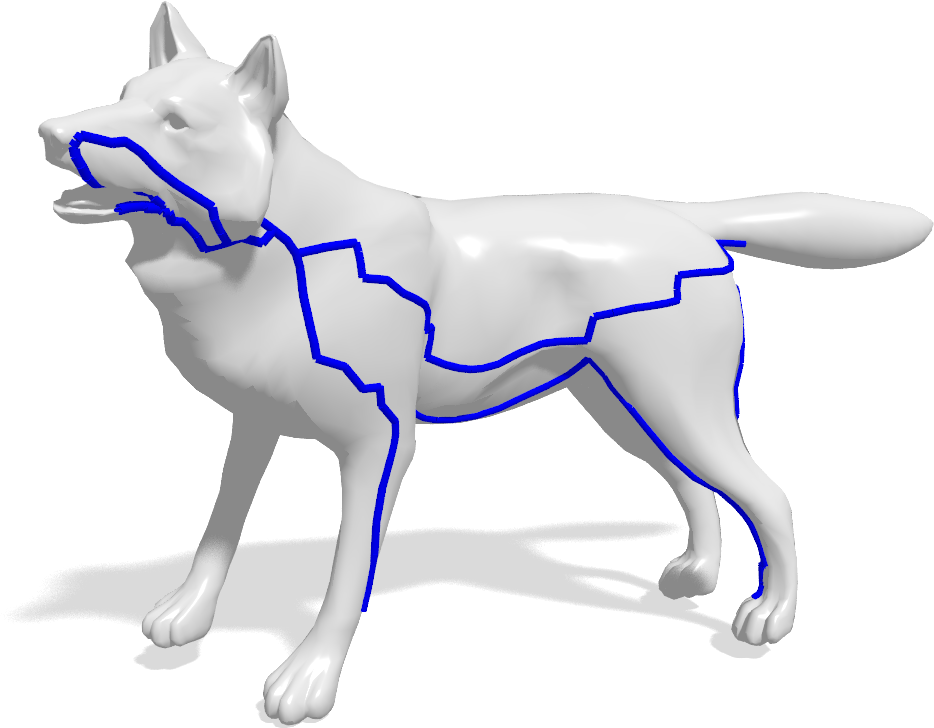}
  \includegraphics[width=0.16\linewidth]{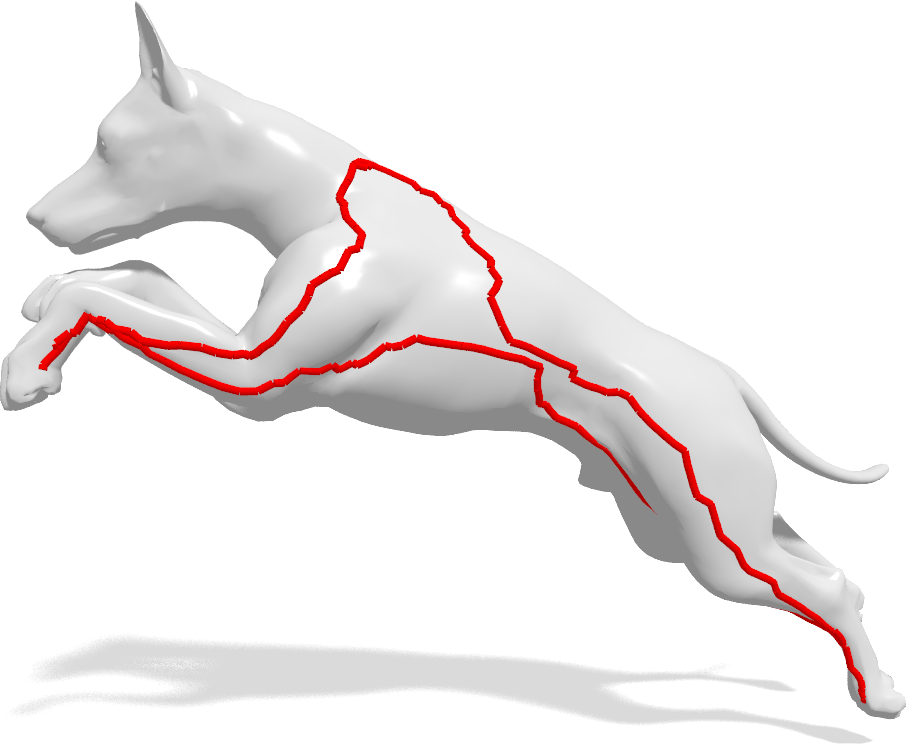}
  \includegraphics[width=0.16\linewidth]{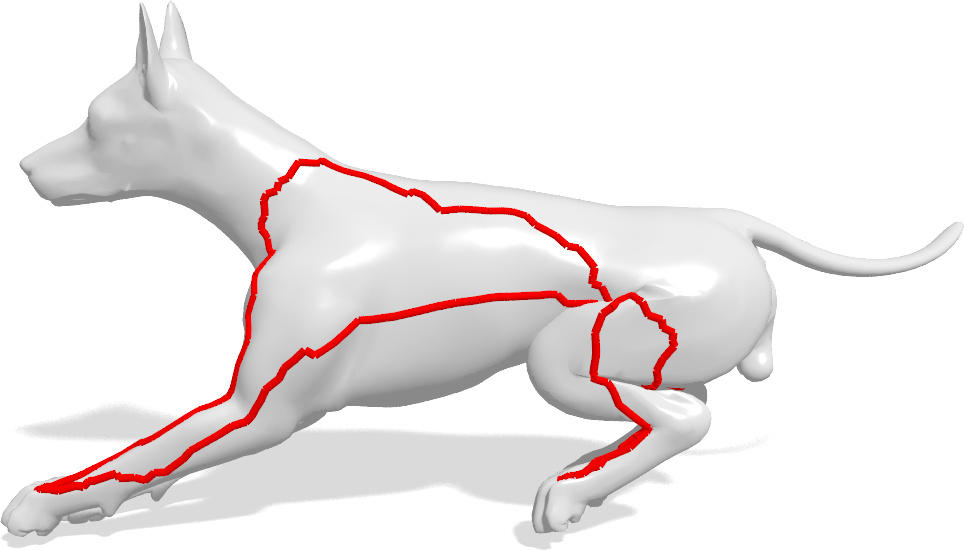}\\
  \includegraphics[width=0.18\linewidth]{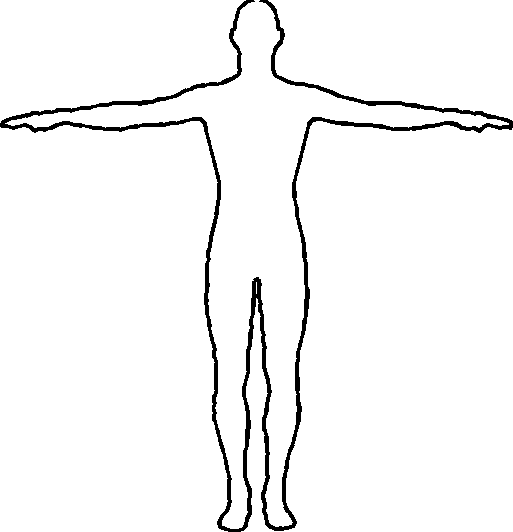}\hspace{20pt}
  \includegraphics[width=0.08\linewidth]{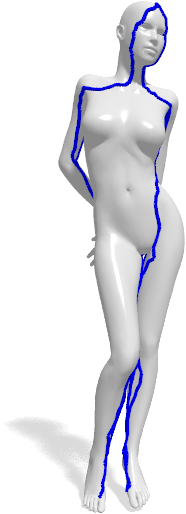}\hspace{5pt}
  \includegraphics[width=0.20\linewidth]{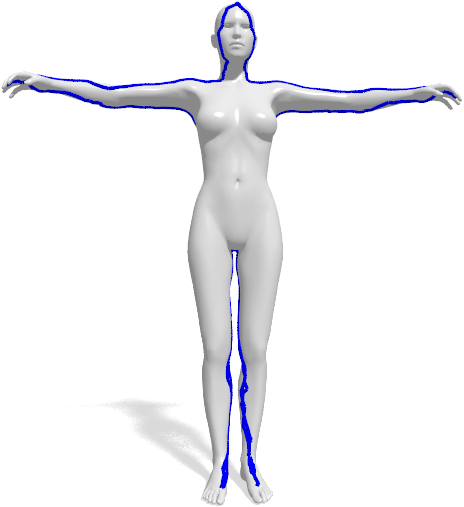}
  \includegraphics[width=0.08\linewidth]{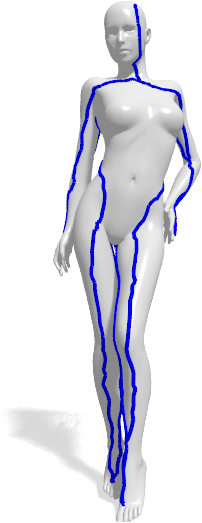}
  \includegraphics[width=0.15\linewidth]{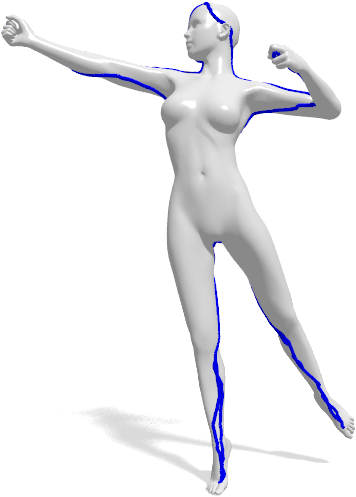}
  \includegraphics[width=0.16\linewidth]{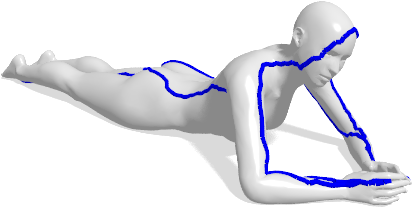}
  \caption{\label{fig:rankings}Retrieval examples on the TOSCA dataset.
    \emph{Left:} Two of the three 2D queries (cat and wolf) have missing
    parts (two legs in contrast to Fig.~\ref{fig:example_matches}). Each
    row shows the top 5 results (ranked by matching energy) provided by
    our method. The corresponding matching curves are shown on top of
    the 3D targets. Note that the dataset only contains 3 wolf shapes,
    which show up as the top 3 matches. The next matches are shapes of
    the class ``dog'', which is a semantically similar class.}
  \end{figure*}

  \section{Conclusion}\label{sec:conclusion}

  We proposed a polynomial-time solution for matching deformable
  planar curves to 3D shapes.  We prove that the worst-case complexity
  of this algorithm is $\cO(mn^2\log(n))$, where $m$ and $n$ denote the
  number of samples on the query curve and the 3D shape respectively.
  We show experimentally that the runtime remains linear with
  respect to $m$, even when employing a branch-and-bound strategy,
  making this a very efficient approach that matches queries with hundreds of
  vertices to 3D shapes with ten thousands of vertices in a few seconds.
  Our algorithm can compute the globally optimal or a $\epsilon$-tight
  solution. The approximate method is more efficient if the features are not discrimative
  enough which causes many iterations in the branch-and-bound strategy.

  Our algorithm provides a powerful tool for shape analysis, and in particular
  has great potential in applications such as 3D shape retrieval from 2D sketches.

  \paragraph{Limitations. }\ The main limitation of our method is the
  assumption that the query 2D shape is a closed planar curve. In
  some situations, this may limit the `expressivity' of the sketch and
  pose a disadvantage compared to image-based approaches in sketch-based
  retrieval applications.
  Furthermore, due to the shortest-path computation the algorithm favors
  solutions having a shorter length on the 3D target shape which are
  not semantically perfect. This makes the correspondence
  not suitable as input for further applications.

  \paragraph{Future research directions. }\ One notable drawback of our
  discrete optimization is the use of Dijkstra's algorithm for finding
  shortest paths on the product manifold, which in some situations may
  not be a consistent discretization of the geodesic distance. In
  follow-up works, we will explore the use of consistent
  fast-marching-like schemes.

\begin{acknowledgement}
    We thank Aneta Stevanovi\'{c} and Matthias Vestner for useful
    discussions. We gratefully acknowledge the support of an Alexander von
    Humboldt Fellowship, the ERC Starting Grant No. 307047 (COMET), the
    ERC Starting Grant 'ConvexVision' and the ERC Consolidator Grant '3D
    Reloaded'.
\end{acknowledgement}

\bibliographystyle{plain}
\bibliography{refs}

\begin{thebibliography}{10}

\bibitem{Appleton-2003}
Ben Appleton and Changming Sun.
\newblock Circular shortest paths by branch and bound.
\newblock {\em Pattern Recognition}, 36(11):2513--2520, 2003.

\bibitem{aubry11}
M.~Aubry, U.~Schlickewei, and D.~Cremers.
\newblock The wave kernel signature: A quantum mechanical approach to shape
  analysis.
\newblock In {\em Proc. ICCV}, pages 1626--1633, 2011.

\bibitem{bogo14}
Federica Bogo, Javier Romero, Matthew Loper, and Michael~J. Black.
\newblock {FAUST}: Dataset and evaluation for {3D} mesh registration.
\newblock In {\em Proc. CVPR}, June 2014.

\bibitem{bronstein2011shape}
A.~Bronstein, M.~Bronstein, L.~Guibas, and M.~Ovsjanikov.
\newblock Shape {G}oogle: Geometric words and expressions for invariant shape
  retrieval.
\newblock {\em Trans. Graphics}, 30(1), 2011.

\bibitem{bronstein08}
Alexander Bronstein, Michael Bronstein, and Ron Kimmel.
\newblock {\em Numerical Geometry of Non-Rigid Shapes}.
\newblock Springer, 2008.

\bibitem{bronstein2006nonrigid3d}
Alexander~M. Bronstein, Michael~M. Bronstein, and Ron Kimmel.
\newblock Efficient computation of isometry-invariant distances between
  surfaces.
\newblock {\em SIAM J. Scientific Computing}, 28/5:1812--1836, 2006.

\bibitem{bronstein2006generalized}
Alexander~M Bronstein, Michael~M Bronstein, and Ron Kimmel.
\newblock Generalized multidimensional scaling: a framework for
  isometry-invariant partial surface matching.
\newblock {\em PNAS}, 103(5):1168--1172, 2006.

\bibitem{dijkstra59}
Edsger~W. Dijkstra.
\newblock A note on two problems in connexion with graphs.
\newblock {\em Numerische Mathematik}, 1:269--271, 1959.

\bibitem{eitz2012sbsr}
Mathias Eitz, Ronald Richter, Tamy Boubekeur, Kristian Hildebrand, and Marc
  Alexa.
\newblock Sketch-based shape retrieval.
\newblock {\em Trans. Graphics}, 31(4):1--10, 2012.

\bibitem{funkhouser2003search}
Thomas Funkhouser, Patrick Min, Michael Kazhdan, Joyce Chen, Alex Halderman,
  David Dobkin, and David Jacobs.
\newblock A search engine for {3D} models.
\newblock {\em Trans. Graphics}, 22(1):83--105, 2003.

\bibitem{furuya14}
Takahiko Furuya and Ryutarou Ohbuchi.
\newblock Hashing cross-modal manifold for scalable sketch-based {3D} model
  retrieval.
\newblock In {\em Proc. 3DV}, 2014.

\bibitem{Lift3D_SA20}
Yulia Gryaditskaya, Felix H\"ahnlein, Chenxi Liu, Alla Sheffer, and Adrien
  Bousseau.
\newblock Lifting freehand concept sketches into 3d.
\newblock {\em ACM Transactions on Graphics (Proc. SIGGRAPH Asia)}, 39, 12
  2020.

\bibitem{herzog15}
Robert Herzog, Daniel Mewes, Michael Wand, Leonidas Guibas, and Hans-Peter
  Seidel.
\newblock {LeSSS}: Learned shared semantic spaces for relating multi-modal
  representations of {3D} shapes.
\newblock {\em Computer Graphics Forum}, 34(5):141--151, 2015.

\bibitem{hueting2015crosslink}
Moos Hueting, Maks Ovsjanikov, and Niloy~J Mitra.
\newblock {CrossLink}: joint understanding of image and {3D} model collections
  through shape and camera pose variations.
\newblock {\em Trans. Graphics}, 34(6):233, 2015.

\bibitem{kovn15}
Artiom Kovnatsky, Michael~M. Bronstein, Xavier Bresson, and Pierre
  Vandergheynst.
\newblock Functional correspondence by matrix completion.
\newblock In {\em Proc. CVPR}, 2015.

\bibitem{LRSBC16}
Zorah L\"ahner, Emanuele Rodol\`a, Frank~R. Schmidt, Michael~M. Bronstein, and
  Daniel Cremers.
\newblock Efficient globally optimal 2d-to-3d deformable shape matching.
\newblock In {\em Proc. of IEEE Conference on Computer Vision and Pattern
  Recognition (CVPR)}, 2016.

\bibitem{leordeanu2005spectral}
Marius Leordeanu and Martial Hebert.
\newblock A spectral technique for correspondence problems using pairwise
  constraints.
\newblock In {\em Proc. ICCV}, 2005.

\bibitem{Li20151}
Bo~Li, Yijuan Lu, Chunyuan Li, et~al.
\newblock A comparison of {3D} shape retrieval methods based on a large-scale
  benchmark supporting multimodal queries.
\newblock {\em CVIU}, 131:1 -- 27, 2015.

\bibitem{Maes-91}
M.~Maes.
\newblock Polygonal shape recognition using string-matching techniques.
\newblock {\em Pattern Recognition}, 24(5):433--440, 1991.

\bibitem{memoli2005theoretical}
Facundo M{\'e}moli and Guillermo Sapiro.
\newblock A theoretical and computational framework for isometry invariant
  recognition of point cloud data.
\newblock {\em Foundations of Computational Mathematics}, 5(3):313--347, 2005.

\bibitem{munkres}
James Munkres.
\newblock Algorithms for the assignment and transportation problems.
\newblock {\em J. SIAM}, 5(1):32--38, 1957.

\bibitem{Nie2021}
Weizhi Nie, Yue Zhao, Jie Niea, An-An Liu, and Sicheng Zhaob.
\newblock Cln: Cross-domain learning network for 2d image-based 3d shape
  retrieval.
\newblock {\em IEEE Transactions on Circuits and Systems for Video Technology},
  2021.

\bibitem{ovs12}
Maks Ovsjanikov, Mirela Ben-Chen, Justin Solomon, Adrian Butscher, and Leonidas
  Guibas.
\newblock Functional maps: A flexible representation of maps between shapes.
\newblock {\em Trans. Graphics}, 31(4):1--11, July 2012.

\bibitem{pokrass13}
Jonathan Pokrass, Alexander~M Bronstein, Michael~M Bronstein, Pablo Sprechmann,
  and Guillermo Sapiro.
\newblock Sparse modeling of intrinsic correspondences.
\newblock {\em Computer Graphics Forum}, 32(2):459--468, 2013.

\bibitem{reuter06}
Martin Reuter, Franz-Erich Wolter, and Niklas Peinecke.
\newblock {Laplace-Beltrami} spectra as 'shape-{DNA}' of surfaces and solids.
\newblock {\em Comput. Aided Design}, 38(4):342--366, April 2006.

\bibitem{rodola12cvpr}
E.~Rodol\`{a}, A.M. Bronstein, A.~Albarelli, F.~Bergamasco, and A.~Torsello.
\newblock A game-theoretic approach to deformable shape matching.
\newblock In {\em Proc. CVPR}, 2012.

\bibitem{rodola-cvpr14}
E.~Rodol\`{a}, S.~Rota Bul\`{o}, T.~Windheuser, M.~Vestner, and D.~Cremers.
\newblock Dense non-rigid shape correspondence using random forests.
\newblock In {\em Proc. CVPR}, 2014.

\bibitem{rodola15}
E.~Rodol\`{a}, L.~Cosmo, M.~M. Bronstein, A.~Torsello, and D.~Cremers.
\newblock Partial functional correspondence.
\newblock {\em Computer Graphics Forum}, 2016.

\bibitem{rodola-cgf14}
E.~Rodol\`{a}, S.~Rota~Bul\`{o}, and D.~Cremers.
\newblock Robust region detection via consensus segmentation of deformable
  shapes.
\newblock {\em Computer Graphics Forum}, 33(5):97--106, 2014.

\bibitem{rodola13elastic}
E.~Rodol\`{a}, A.~Torsello, T.~Harada, Y.~Kuniyoshi, and D.~Cremers.
\newblock Elastic net constraints for shape matching.
\newblock In {\em Proc. ICCV}, 2013.

\bibitem{Schmidt-et-al-iccv07}
F.~R. Schmidt, D.~Farin, and D.~Cremers.
\newblock Fast matching of planar shapes in sub-cubic runtime.
\newblock In {\em Proc. ICCV}, 2007.

\bibitem{Schmidt-et-al-cvpr09}
F.~R. Schmidt, E.~T{\"o}ppe, and D.~Cremers.
\newblock Efficient planar graph cuts with applications in computer vision.
\newblock In {\em Proc. CVPR}, June 2009.

\bibitem{Schoenemann-Cremers-pami10}
T.~Schoenemann and D.~Cremers.
\newblock A combinatorial solution for model-based image segmentation and
  real-time tracking.
\newblock {\em Trans. PAMI}, 32(7):1153--1164, 2010.

\bibitem{shewchuk02}
Jonathan~Richard Shewchuk.
\newblock {D}elaunay refinement algorithms for triangular mesh generation.
\newblock {\em Comput. Geom. Theory Appl.}, 22(1-3):21--74, 2002.

\bibitem{su15mvcnn}
Hang Su, Subhransu Maji, Evangelos Kalogerakis, and Erik~G. Learned{-}Miller.
\newblock Multi-view convolutional neural networks for {3D} shape recognition.
\newblock In {\em Proc. ICCV}, 2015.

\bibitem{sun09}
Jian Sun, Maks Ovsjanikov, and Leonidas Guibas.
\newblock A concise and provably informative multi-scale signature based on
  heat diffusion.
\newblock In {\em Proc. SGP}, pages 1383--1392, 2009.

\bibitem{TABIA201724}
Hedi Tabia and Hamid Laga.
\newblock Learning shape retrieval from different modalities.
\newblock {\em Neurocomputing}, 253:24--33, 2017.

\bibitem{tangelder2008survey}
Johan~WH Tangelder and Remco~C Veltkamp.
\newblock A survey of content based {3D} shape retrieval methods.
\newblock {\em Multimedia Tools and Applications}, 39(3):441--471, 2008.

\bibitem{van2011survey}
Oliver Van~Kaick, Hao Zhang, Ghassan Hamarneh, and Daniel Cohen-Or.
\newblock A survey on shape correspondence.
\newblock {\em Computer Graphics Forum}, 30(6):1681--1707, 2011.

\bibitem{10.1145/3240508.3240699}
Lingjing Wang, Cheng Qian, Jifei Wang, and Yi~Fang.
\newblock Unsupervised learning of 3d model reconstruction from hand-drawn
  sketches.
\newblock Association for Computing Machinery, 2018.

\bibitem{wind11}
T.~Windheuser, U.~Schlickewei, F.~R. Schmidt, and D.~Cremers.
\newblock Geometrically consistent elastic matching of {3D} shapes: A linear
  programming solution.
\newblock In {\em Proc. ICCV}, 2011.

\bibitem{Xu_2019_ICCV}
Cheng Xu, Zhaoqun Li, Qiang Qiu, Biao Leng, and Jingfei Jiang.
\newblock Enhancing 2d representation via adjacent views for 3d shape
  retrieval.
\newblock In {\em Proceedings of the IEEE/CVF International Conference on
  Computer Vision (ICCV)}, October 2019.

\bibitem{Yang20}
Song Yang.
\newblock Sketch-based 3d shape retrieval with multi-silhouette view based on
  convolutional neural networks.
\newblock In {\em 2020 13th International Conference on Intelligent Computation
  Technology and Automation (ICICTA)}, pages 223--226, 2020.

\bibitem{ZHOU2017101}
Yan Zhou and Fanzhi Zeng.
\newblock 2d compressive sensing and multi-feature fusion for effective 3d
  shape retrieval.
\newblock {\em Information Sciences}, 409-410:101--120, 2017.

\end{thebibliography}

\end{document}